%% file: preprint.tex
\documentclass{article}
\usepackage{chapterbib}
\usepackage[round,sectionbib]{natbib}
\usepackage{amsmath}
\usepackage{amssymb}
\usepackage{amsthm}
\usepackage{authblk}
\usepackage{booktabs}
\usepackage{dutchcal}
\usepackage[margin=1in]{geometry}
\usepackage{hyperref}
\usepackage[capitalize]{cleveref}
\usepackage{mathrsfs}
\usepackage{mathtools}
\usepackage{subcaption}
\usepackage{tikz}
\usetikzlibrary{3d,arrows,calc,patterns,positioning,snakes}
\usepackage{xcolor}
\DeclareMathOperator*{\argmin}{argmin}

\newcommand{\E}[2]{\mathop{\mathbb{E}}_{#1} \left[ #2 \right]}
\newcommand{\ent}[1]{\mathrm{H} \left( #1 \right)}
\newcommand{\kld}[2]{\mathrm{D_{KL}} \left( #1 \left| \left| #2 \right. \right. \right)}

\newtheorem{assumption}{Assumption}
\crefname{assumption}{Assumption}{Assumptions}
\newtheorem{corollary}{Corollary}
\newtheorem{definition}{Definition}
\newtheorem{lemma}{Lemma}
\newtheorem{theorem}{Theorem}
\newtheorem{remark}{Remark}
\newtheorem{proposition}{Proposition}

\newcommand{\datasample}{x}
\newcommand{\datasamples}{X}
\newcommand{\dataval}{\mathbf{x}}
\newcommand{\datavals}{\mathbf{X}}
\newcommand{\dist}{P}
\newcommand{\distvals}{\Delta_{\datasamples}}
\newcommand{\bary}[1]{\overline{#1}}
\newcommand{\bestfitdist}[1]{P_{\star}}
\newcommand{\bias}{\mathcal{B}}
\newcommand{\convergence}{\mathcal{C}}
\newcommand{\distshift}{\mathcal{D}}
\newcommand{\distshiftLearner}{\mathcal{D}_{\mathrm{learner}}}
\newcommand{\epistemicerr}{\mathbf{e}}
\newcommand{\eps}{\epsilon}
\newcommand{\hyperdist}{\mathcal{Q}}
\newcommand{\hyperdistpdf}{\mathcal{q}}
\newcommand{\hyperdistvals}{\Delta\left( \Delta_{\datasamples} \right)}
\newcommand{\hypersourcedist}{\hyperdist^{S}}
\newcommand{\hypertargetdist}{\hyperdist^{T}}
\newcommand{\hypersourcepdf}{\hyperdistpdf^{S}}
\newcommand{\hypertargetpdf}{\hyperdistpdf^{T}}
\newcommand{\loss}{\ell}
\newcommand{\margin}{\alpha}
\newcommand{\marginprob}{\delta}
\newcommand{\maxp}{\overline{u}}
\newcommand{\minp}{\underline{u}}
\newcommand{\modelclass}{\pi}
\newcommand{\preddist}{\widehat{P}}
\newcommand{\preddistpdf}{\widehat{p}}
\newcommand{\repl}[1]{#1^{\prime}}
\newcommand{\Rplus}{\mathbb{R}^+}
\newcommand{\sourcedata}{\dataval_{(1:n)}}
\newcommand{\task}{Q}
\newcommand{\sourcetask}{\task^{s}}
\newcommand{\targettask}{\task^{t}}
\newcommand{\taskdist}{Q}
\newcommand{\sourcetaskdist}{\taskdist^s}
\newcommand{\support}[1]{\mathrm{supp}\left( #1 \right)}
\newcommand{\targettaskdist}{\taskdist^t}
\newcommand{\targettaskpdf}{q^{t}}
\newcommand{\thetaval}{\vartheta}
\newcommand{\thetavals}{\Theta}

\newcommand{\trueprob}[1]{\mathrm{Pr}\left( #1 \right)}
\newcommand{\tv}[2]{\mathrm{d_{TV}}\left( #1, #2 \right)}
\newcommand{\tvset}{\tau_{\hyperdist}}
\newcommand{\var}[1]{\mathbb{V}_{\dataval} \left[ #1 \right]}
\newcommand{\maxvar}[1]{\mathcal{V}\left[ #1 \right]}
\newcommand{\wpnorm}[1]{\sqrt{\var{#1}}}

\title{Epistemic Errors of Imperfect Multitask Learners When Distributions Shift}
\author[1]{Sabina J. Sloman\footnote{Correspondence to \texttt{sabina.sloman@manchester.ac.uk}.}}
\author[1]{Michele~Caprio}
\author[1,2,3]{Samuel Kaski}
\affil[1]{Department of Computer Science, University of Manchester, Manchester, UK}
\affil[2]{ELLIS Institute Finland, Helsinki, Finland}
\affil[3]{Department of Computer Science, Aalto University, Helsinki, Finland}

\begin{document}

\include{main}
\appendix
\include{appendix}

\end{document}

%% file: main.tex
\maketitle

\begin{abstract}
    Uncertainty-aware machine learners, such as Bayesian neural networks, output a quantification of uncertainty instead of a point prediction.
    We provide uncertainty-aware learners with a principled framework to characterize, and identify ways to eliminate, errors that arise from reducible (\textit{epistemic}) uncertainty.
    We introduce a principled definition of \textit{epistemic error}, and provide a decompositional \textit{epistemic error bound} which operates in the very general setting of imperfect multitask learning under distribution shift.
    In this setting, the training (\textit{source}) data may arise from multiple tasks, the test (\textit{target}) data may differ systematically from the source data tasks, and/or the learner may not arrive at an accurate characterization of the source data.
    Our bound separately attributes epistemic errors to each of multiple aspects of the learning procedure and environment.
    As corollaries of the general result, we provide epistemic error bounds specialized to the settings of Bayesian transfer learning and distribution shift within $\epsilon$-neighborhoods.
\end{abstract}

\section{INTRODUCTION}\label{sec:intro}
    Probabilistic and uncertainty-aware machine learning paradigms have gained increasing traction in recent years \citep{izbicki_machine_2025}.
    Methods like Bayesian inference \citep{papamarkou_position_2024}, conformal prediction \citep{angelopoulos_conformal_2023}, and credal learning \citep{caprio_cbdl_2024} are used in the context of models of increasing scale.
    Instead of a point prediction, these methods provide a quantification of their uncertainty about an outcome, in the form of a predictive distribution, prediction interval, or set.

    Statistical learning theory provides machine learners with a framework to characterize prediction errors.
    For uncertainty-aware learners, prediction errors are an irrelevant, incomplete, or uninformative measure of performance: They do not capture the learner's additional goal of uncertainty quantification, and conflate error that arises from reducible and irreducible sources of uncertainty.
    In this work, we develop a framework and collection of results that bring the conceptual tools of statistical learning theory to uncertainty-aware paradigms.
    Our goal is to provide uncertainty-aware learners with a principled framework to characterize, and identify ways to improve, their performance.

    Learners confront two types of uncertainty \citep{hullermeier_aleatoric_2021}.
    Irreducible, \textit{aleatoric uncertainty} refers to the amount of noise inherent to the data-generating process (DGP).
    On the other hand, the learner can reduce its \textit{epistemic uncertainty} by taking actions to acquire knowledge of the DGP.

    We say a learner commits \textit{epistemic errors} when it is mistaken about the nature of the DGP.
    An uncertainty-aware learner who eliminates epistemic errors is both accurate and calibrated: It is accurate in the sense that it captures the DGP, and is calibrated in the sense that its uncertainty reflects the variability in outcomes it encounters at test-time.

    Our main contribution is the concept of an \textit{epistemic error bound}: a guarantee of how well a learner captures the DGP.
    Our second, technical contribution is a general, ``decompositional'' epistemic error bound (\Cref{thm:imperfect-distshift}).
    Our bound is general in the sense that it accommodates any of several sources of epistemic uncertainty: Training data may be a combination of data from multiple tasks (\textit{multitask learning}),
    the task in which the learner wants to predict may differ systematically from the training data (\textit{distribution shift}), or
    the learner may \textit{imperfectly} characterize the training data.
    Our bound is decompositional in the sense that it decomposes into three separate contributors to epistemic error: model restrictions ($\bias$), data scarcity ($\convergence$), and distribution shift ($\distshift$).
    
    To summarize, our bound has the following structure:
    $$\trueprob{ \epistemicerr \geq \alpha + \bias + \convergence + \distshift } \leq \delta$$
    where $\epistemicerr$ is the epistemic error and $\alpha \in \mathbb{R}^{+}$ is selected to determine the value of the bound (or \textit{epistemic error margin}; see \Cref{sec:neighbors}).
    Our technical innovations are to provide (i) an expression for $\marginprob$, and (ii) suitable definitions for $\bias$, $\convergence$, and $\distshift$.

    Taken together, these contributions provide uncertainty-aware learners with a principled framework within which to identify which errors can be reduced, and how to reduce them.

    \vspace{2mm}
    \noindent \textit{Example.} 
    To provide uncertainty quantification, methods like Bayesian neural networks have been applied in healthcare settings \citep{peng_bayesian_2019,abdar_uncertainty_2021,abdullah_review_2022}.
    When using such a model to predict healthcare outcomes in a specified target location, one might consider the following questions:
    (Q1) Should I train my model on abundant data from other, somewhat dissimilar, locations, or only on scarce data that arise from the target location?
    (Q2) Should I use a deep neural network that has the potential to accurately characterize the training data, but may require more data to train than a shallower network?

    Each of these questions expresses a trade-off between reducing the sources of epistemic error that are represented in our bound.
    (Q1) captures the trade-off between data scarcity ($\convergence$) and distribution shift ($\distshift$): Abundant training data contribute to convergence on a good approximation of those data, while data from locations that are dissimilar to the target location contribute to distribution shift.
    (Q2) captures the trade-off between model restrictions ($\bias$) and data scarcity ($\convergence$): Highly parameterized models reduce approximation bias, but at the cost of requiring potentially infeasibly large amounts of training data.
    Our work provides a framework within which the modeler can reason about the best setting for their clinical prediction task.

    \paragraph{Structure of the paper}
        In the remainder of this section, we provide a brief overview of statistical learning theory.\footnote{
            \Cref{ap:ast} compares our framework to the paradigm of asymptotic statistical theory.
        }
        \Cref{sec:formulation} introduces a very general multitask learning setting that allows for imperfect learning and distribution shift.
        \Cref{sec:neighbors} provides our main conceptual contribution: the concept of an epistemic error bound.
        \Cref{sec:result} provides our main result: a decompositional epistemic error bound for the general setting introduced in \Cref{sec:formulation}.
        \Cref{sec:cases} tailors the general bound to each of the settings of Bayesian transfer learning and distribution shift within total variation neighborhoods.
        \Cref{sec:discussion} discusses limitations of our work and directions for future research.

\paragraph{Statistical learning theory}\label{sec:related-work}
    (SLT) is concerned with the development of \textit{generalization bounds} on the expected loss that a learner will incur when encountering a given DGP.
    Probably approximately correct (PAC) learning theory, a paradigm of SLT, provides probabilistic generalization bounds (so-called PAC bounds) of the form
    $$\trueprob{L \geq \widehat{L} + \epsilon} \leq \delta$$
    where $L$ is the expected loss with respect to a distribution of data, $\widehat{L}$ is an empirical estimate of $L$, $\epsilon$ is a small positive constant, and $\delta \in (0,1)$ \citep{shalev-shwarz_understanding_2014}.
    Our concept of an epistemic error bound is inspired by such PAC bounds.
    Like a PAC bound, an epistemic error bound is a probabilistic upper bound.
    An epistemic error bound differs from a PAC bound in two key ways: While the loss that appears in a generalization bound may arise from epistemic errors or aleatoric factors, an epistemic error bound isolates the amount of reducible error.
    In addition, our main bound in \Cref{thm:imperfect-distshift} applies in a setting more general than those studied in other statistical learning theoretic paradigms.\footnote{\Cref{sec:distance} provides, as corollaries to our result, generalization bounds for common loss functions that apply in this general setting.}
    In \Cref{ap:related-work}, we discuss and contrast with our work the following learning theoretic paradigms:
    \textbf{Multitask domain generalization}
        provides generalization bounds for the multitask setting \citep{baxter_model_2000,maurer_bounds_2006,maurer_sparse_2013,maurer_excess_2013,liu_algorithm_2017,deshmukh_generalization_2019}.
    \textbf{Domain adaptation theory}
        provides generalization bounds that apply in certain cases of distribution shift \citep{redko_survey_2019}.
    \textbf{Credal learning theory}
        extends SLT with a model of uncertainty about the DGP that uses imprecise probabilities \citep{caprio_credal_2024}.

\section{PROBLEM FORMULATION}\label{sec:formulation}
    Each data point is an element $\datasample$ of space $\datasamples$.
    Measurable events are denoted by $\dataval \in \datavals$, where $\datavals$ is a $\sigma$-algebra on $\datasamples$.
    We denote by $\distvals$ the space of probabilities on $\datasamples$, and by $\hyperdistvals$ the space of second-order distributions on $\datasamples$, that is, the space of distributions over $\distvals$.
    
    Our formulation characterizes both the learner's beliefs and the learning environment.
    We thus require notation to distinguish between, on one hand, the  data-generating processes (DGPs) --- the distributions that generate the outcomes the learner will in fact encounter --- and, on the other, the learner's (possibly misspecified) beliefs about these distributions.
    To refer to the DGPs, we use $Q$ and $q$ to denote first-order probability distributions and their probability density functions (p.d.f.'s), respectively\footnote{In case $\vert X \vert \in \mathbb{N}$, i.e., of a finite sample space, $q$ is the probability mass function (p.m.f.) of $Q$.
        Statements that reference a p.d.f. $q$ can be trivially generalized to the case of a finite sample space.
    }, and $\hyperdist$ and $\hyperdistpdf$ for second-order probability distributions and their densities, respectively\footnote{In case $\hyperdist$ has finite support, $\hyperdistpdf$ is the p.m.f. of $\hyperdist$.}.
    To refer to the learner's beliefs, we use $P$ and $p$ to denote first-order probability distributions and their densities\footnote{In case of a finite sample space, $p$ is the p.m.f. of $P$.}, respectively.
    We use $\trueprob{\mathbf{a}}$ to refer to the true probability (determined by the DGPs and unavailable to the learner) that event $\mathbf{a}$ occurs.
    We use $\mathrm{supp}(Q)$ to denote the support of distribution $Q$.
        
    In a classification problem where the sample space $\datasamples$ is finite and $\datavals = \mathcal{P}(\datasamples)$, the power set of $X$, a set $\dataval \in \datavals$ can be interpreted as a plausible data set and $\trueprob{ \dataval }$ as the probability of encountering an element of that data set.\footnote{In this case, $\dataval$ can be considered a multiset and $\datavals$ the space of multisets over $\datasamples$.}
    In a regression problem where the sample space $\datasamples = \mathbb{R}$ and $\datavals$ is the Borel $\sigma$-algebra on $\mathbb{R}$, a set $\dataval \in \datavals$ can be interpreted as an interval of possible values a data point can take and $\trueprob{ \dataval }$ as the probability a data point falls within that interval.

    \paragraph{Multitask learning}
        In the multitask learning setting, the learner encounters data from multiple \textbf{tasks}, or distinct DGPs.
        A \textbf{task} $\task$ is a first-order distribution from which data points are sampled.
        We assume tasks themselves are samples from a second-order \textbf{task distribution} $\hyperdist \in \hyperdistvals$.

        A task distribution $\hyperdist$ induces a first-order distribution over the data themselves.
        We characterize the expectation and variability of this first-order distribution using the following definitions of the \textit{barycenter} and \textit{variability} of $\hyperdist$.
        \begin{definition}[Barycenter]\label{def:bary}
            The barycenter $\bary{\hyperdist} \in \distvals$ of a task distribution $\hyperdist \in \hyperdistvals$ having p.d.f. $\hyperdistpdf$ assigns to $\dataval \in \datavals$ a probability
            $$\bary{\hyperdist}(\dataval) \coloneqq \int_{\distvals} \taskdist(\dataval) ~ \hyperdistpdf(\taskdist) ~ \text{\em d}\taskdist = \E{\task \sim \hyperdist}{\taskdist(\dataval)}.$$
        \end{definition}

        \begin{definition}[Variability]\label{def:var}
            The variance at $\dataval\in \datavals$ of a task distribution $\hyperdist \in \hyperdistvals$ having p.d.f. $\hyperdistpdf$ is
            $$\var{\hyperdist} \coloneqq \int_{\distvals} \left[ \taskdist(\dataval) - \overline{\hyperdist}(\dataval) \right]^2 ~ \hyperdistpdf(\taskdist) ~ \text{\em d}\taskdist.$$
            The variability of $\hyperdist$ is the maximum variance with respect to $\dataval \in \datavals$, i.e., is
            $$\maxvar{\hyperdist} \coloneqq \sup_{\dataval \in \datavals}\var{\hyperdist}.$$
        \end{definition}

        We use \textit{multitask distribution} to refer to a task distribution under which tasks disagree about the probability of encountering any $\dataval \in \datavals$:
        \begin{definition}[Multitask distribution]\label{def:multitask}
            We say a task distribution $\hyperdist$ is a multitask distribution if $\forall \dataval \in \datavals$, $\var{\hyperdist} > 0$.
        \end{definition}
        \begin{remark}\label{rem:var}
            Notice that the variance and variability of a multitask distribution are bounded from both below and above ($0 < \var{\hyperdist} \leq \maxvar{\hyperdist} \leq 1$) by the fact that $0 \leq \taskdist(\dataval) \leq 1$ for any $\dataval \in \datavals$ and $\taskdist \in \distvals$.
        \end{remark}

        During \textit{training}, the learner encounters data in each of multiple \textbf{source tasks}.
        We assume the source tasks are sampled independently and identically from a \textbf{source task distribution} $\hypersourcedist$ having p.d.f. $\hypersourcepdf$, and so the probability that the learner encounters a $\datasample \in \dataval$ among the source data is
        \begin{align*}
            \trueprob{ \dataval } = \int_{\distvals} \taskdist(\dataval) ~ \hypersourcepdf \left( \taskdist \right) ~ \text{d}\taskdist.
        \end{align*}

        Importantly, notice that $\trueprob{ \dataval } = \bary{\hypersourcedist}\left( \dataval \right)$ (\Cref{def:bary}). In other words, $\bary{\hypersourcedist}$ can be thought of as the true source data distribution.

        During \textit{testing}, the learner encounters a target task $\targettaskdist$.
        The learner's goal is to predict the outcomes that will accrue to $\targettask$.
        We assume the target task is sampled from a \textbf{target multitask distribution} $\hypertargetdist$.

        We say \textbf{distribution shift} has occurred when $\hypersourcedist \neq \hypertargetdist$, i.e., the source tasks are sampled from a different distribution than the target task.

        The learner begins with a set of possible models $\modelclass$, and selects from this set a predictive distribution $\preddist \in \modelclass$ with a corresponding p.d.f. $\preddistpdf$.
        Epistemic (test-time) error can be thought of as the degree to which $\preddist$ differs from $\targettaskdist$, a notion we make more precise in \Cref{sec:neighbors}.
        We say learning is \textbf{imperfect} when $\preddist \neq \bary{\hypersourcedist}$, i.e., the learner's predictions differ from the source data distribution.

    \paragraph{Uncertainty-aware paradigms}
        We briefly discuss whether and how our setting characterizes the three uncertainty-aware paradigms mentioned in \Cref{sec:intro}.

        \textit{Bayesian learners} (such as Bayesian neural networks) derive a posterior parameter distribution and output a posterior predictive distribution $\preddist$ over the space $\datasamples$.
        We discuss this paradigm in more detail in \Cref{sec:cases}.

        \textit{Conformal predictors} output an interval $C(\tilde{z})$, a set predicted to contain the true label $\tilde{y}$ of a target instance $\tilde{z}$ with high probability \citep{angelopoulos_conformal_2023}.
        Here, $\preddist$ is a discrete distribution over a sample space consisting of two elements: the outcome where the true label $\tilde{y}$ is contained in $C(\tilde{z})$, and the outcome where it is not.
        This case violates the conditions of our setting: $\preddist$ does not share a sample space with the data distributions.
        We consider an extension of our work to the CP setting an avenue for future work.

        \textit{Credal Bayesian deep learners} draw from both the Bayesian and CP perspectives, outputting a set of posterior predictive distributions \citep{caprio_cbdl_2024}.
        Thus, both the above interpretations of $\preddist$ can apply to this case.

\section{EPISTEMIC ERROR BOUND}\label{sec:neighbors}
    We here introduce the concept of an epistemic error bound.
    We first provide a formal definition of epistemic (test-time) error (\Cref{def:epistemic}), which leverages the total variation (TV) distance.
    \begin{definition}[Total variation (TV) distance ($\mathrm{d_{TV}}$)]\label{def:tv}
        The total variation (TV) distance between two distinct distributions $\dist$ and $\repl{\dist}$ is
        $$\tv{\dist}{\repl{\dist}} \coloneqq \sup_{\mathbf{a} \in \mathbf{A}} \left| 
        \dist(\mathbf{a}) - \repl{\dist}(\mathbf{a}) \right|$$
        where $\mathbf{A}$ is the $\sigma$-algebra over the shared support of $\dist$ and $\repl{\dist}$.
    \end{definition}
    
    We define \textbf{epistemic }(test-time) \textbf{error} (henceforth, epistemic error) as follows.
    \begin{definition}[Epistemic error]\label{def:epistemic}
        The learner's epistemic error in characterizing a target task distribution $\targettaskdist$ with predictor $\preddist$ is
        $$\epistemicerr \coloneqq \tv{\preddist}{\targettaskdist}.$$
    \end{definition}
    In other words, epistemic error is a special case of \Cref{def:tv} where $\dist = \preddist$, $\repl{\dist} = \targettaskdist$, and $\mathbf{A}=\datavals$.
    
    The TV distance has several desirable characteristics, satisfying the properties of both an $f$-divergence and integral probability metric \citep{birrell_divergences_2022}.
    However, our concept of an epistemic error bound does not rely on a specific definition of epistemic error, and our results extend immediately to any function of $\preddist$ and $\targettask$ that is a lower bound on some function of the TV distance.
    In \Cref{sec:distance}, we state our main result using definitions of epistemic error in terms of other measures.\footnote{See also \citet{sale_second-order_2024} for a thorough investigation of the suitability of various divergence measures of related concepts, such as epistemic uncertainty.}

    Equipped with a formal definition of epistemic error, we now introduce the concept of an epistemic error bound.
    \begin{definition}[Epistemic error bound]
        Given a predictor $\preddist \in \modelclass$ and a target task distribution $\hypertargetdist$, if
        $$\trueprob{ \epistemicerr \geq \margin } \leq \marginprob,$$
        we say the learner experiences an \textbf{epistemic error margin} greater than $\margin$ with probability at most $\marginprob$.
        The bound is probabilistic because the target task is a random quantity, i.e., the statement reflects stochasticity in the realization of $\targettask \sim \hypertargetdist$.
    \end{definition}

\section{GENERAL RESULT}\label{sec:result}
    Our main result is \Cref{thm:imperfect-distshift}: an epistemic error bound for the multitask learning setting, which applies in the presence of distribution shift and if learning is imperfect.
    Our results make the following assumption:
    \begin{assumption}\label{as:closed}
        A task distribution $\hyperdist$ satisfies this assumption if $\forall \taskdist \in \support{\hyperdist}$, $\tau_{\hyperdist}\left( \taskdist \right) \coloneqq \left\{ \left| \taskdist\left( \dataval \right) - \bary{\hyperdist}\left( \dataval \right) \right| : ~ \dataval \in \datavals \right\}$ is a closed subset of the unit interval.
    \end{assumption}
    \begin{remark}\label{rem:closed}
        If \Cref{as:closed} holds, the epistemic error is practically attainable, in the sense that there is some $\dataval \in \datavals$ for which $\epistemicerr = \left| \preddist(\dataval) - \targettask(\dataval) \right|$ (\Cref{prop:sup-attain} in \Cref{ap:proof-multitask}).
        Notice that it always holds in the case of a finite sample space.
    \end{remark}

    We impose \Cref{as:closed} primarily for clarity of exposition: If it is violated, all our results can be modified to absorb an arbitrarily small constant into the epistemic error margin.
    \Cref{lem:multitask,cor:not-closed} in \Cref{ap:proofs} provide statements of the result in which the assumption is leveraged and of our main result, respectively, when \Cref{as:closed} is violated.

    \paragraph{Outline of results}
        While we defer complete proofs to \Cref{ap:proofs}, we here provide an outline of our proof techniques.
        Conceptualizing possible data distributions in a metric space endowed with the TV distance metric, the length of a path $P_0 \rightarrow P_1 \rightarrow P_2$ can be interpreted as the cumulative error in using $P_0$ to represent $P_1$ and in using $P_1$ to represent $P_2$.
        If a learner originates at $P_0$ and considers $P_2$ its target destination, an immediate consequence of the triangle inequality is that it is better, in terms of error, to travel directly to $P_2$ than to ``stop over'' at $P_1$.

        We construct \Cref{thm:imperfect-distshift} iteratively, providing standalone results in progressively more general settings.
        \Cref{sec:multitask} provides an epistemic error bound in a multitask learning setting where learning is perfect and distribution shift absent.
        \Cref{sec:imperfect,sec:indist} relax these conditions: \Cref{sec:imperfect} allows for imperfect learning, and \Cref{sec:indist} additionally allows for distribution shift.
        In these sections, we apply the triangle inequality to bound the learner's distance from the source data distribution (as a function of model restrictions $\bias$ and data scarcity $\convergence$) and the distance from the source to target data distribution (i.e., the amount of distribution shift $\distshift$), respectively.

        In some cases the epistemic error margin may be considerably loose (at the extreme, if we know nothing about where the learner originates, we cannot know how long its journey will be).
        In \Cref{sec:negtransfer}, we discuss an interpretation of the amount of looseness in the margin in terms of the phenomenon of negative transfer.
        
    \subsection{Multitask Learning}\label{sec:multitask}
        We begin by assessing the additional risk that the learner takes on by learning on the basis of data from multiple, as opposed to a single, source tasks.
        \Cref{cor:perfect-no-distshift} provides an epistemic error bound for the \textit{perfect learner in the absence of distribution shift}.
        
        \begin{lemma}[Epistemic error depends on task variability]\label{cor:perfect-no-distshift}
            Given a predictor $\preddist$, a source multitask distribution $\hypersourcedist$ satisfying \Cref{as:closed}, $\preddist = \bary{\hypersourcedist}$ (\textbf{perfect learning}), $\hypertargetdist = \hypersourcedist$ (\textbf{no distribution shift}), and any $\margin \in \Rplus$,
            $$\trueprob{ \epistemicerr \geq \margin } \leq \frac{\maxvar{\hypertargetdist}}{\margin^2}.$$
        \end{lemma}
        \Cref{cor:perfect-no-distshift} tells us that the learner is more likely to achieve a given margin when target tasks are less variable.
        It is a consequence of a concentration inequality which controls how much a source task can deviate from the source data distribution as a function of source task variability $\maxvar{\hypersourcedist}$ (\Cref{ap:proof-multitask}).\footnote{
            Notice however that the result in \Cref{cor:perfect-no-distshift} is non-vacuous only if $\margin > \sqrt{\maxvar{\hypertargetdist}}$.
        }

    \subsection{Imperfect Multitask Learning}\label{sec:imperfect}
        We now extend the result in \Cref{cor:perfect-no-distshift} to the imperfect learning setting.
        Recall from \Cref{sec:formulation} that imperfect learning occurs when the learner's predictive distribution imperfectly characterizes the source data distribution ($\preddist \neq \bary{\hypersourcedist}$).
        Imperfect learning can occur because of limitations of the learner (e.g., model restrictions) or of the data it learns from (e.g., data scarcity).
        To capture these distinct sources of imperfect learning, we differentiate between the learner's \textit{approximation bias} and \textit{lack of convergence}.
        
        Approximation bias ($\bias$) captures how well the class of models $\modelclass$ entertained by the learner can in principle approximate the source data distribution.
        \begin{definition}[Approximation bias ($\bias$)]\label{def:bias}
            The degree of approximation bias is
            $$\bias \coloneqq \tv{\bestfitdist{S}}{\bary{\hypersourcedist}},$$
            where
            \begin{equation}\label{eq:bestfitdist}
                \bestfitdist{S} \coloneqq \argmin_{P \in \modelclass}{\tv{P}{\bary{\hypersourcedist}}}
            \end{equation}
            is the \textbf{best approximation} within the learner's model class.
        \end{definition}
        In the perfect learning setting, \Cref{eq:bestfitdist} evaluates to 0 for $\preddist = \bary{\hypersourcedist}$ (there is no approximation bias).
        The degree of approximation bias does not depend on properties of the specific source data (such as the sample size) or on the specific optimization principle or algorithm used.

        Lack of convergence ($\convergence$) captures how far the learner's predictor $\preddist$ is to the best approximation $\bestfitdist{S}$.
        \begin{definition}[Lack of convergence ($\convergence$)]\label{def:convergence}
            The degree of lack of convergence is
            $$\convergence \coloneqq \tv{\preddist}{\bestfitdist{S}}.$$
        \end{definition}
        In general, the degree of convergence does depend on the specific source data and optimization procedure.
        \Cref{def:convergence} is reminiscent of results from asymptotic statistical theory on, for instance, the convergence rates of Bayesian parameter distributions (\citealp{doob_application_1949,van-der-vaart_asymptotic_2000}; see also discussion in \Cref{ap:ast} and results in \Cref{sec:cases}).
        
        \Cref{lem:imperfect-no-distshift} provides an epistemic error bound for the \textit{imperfect learner in the absence of distribution shift}.
        \begin{lemma}[Epistemic error depends on task variability, model restrictions, and data scarcity]\label{lem:imperfect-no-distshift}
             Given a predictor $\preddist$, a source multitask distribution $\hypersourcedist$ satisfying \Cref{as:closed}, $\hypertargetdist = \hypersourcedist$ (\textbf{no distribution shift}), and any $\margin \in \Rplus$,
             $$\trueprob{ \epistemicerr \geq \margin + \bias + \convergence } \leq \frac{\maxvar{\hypertargetdist}}{\margin^2}.$$
        \end{lemma}
        \Cref{lem:imperfect-no-distshift} shows that in the case of imperfect learning, the epistemic error margin absorbs the degrees of approximation bias and lack of convergence.
        The consequence of allowing the learner's ``journey'' to the target task to originate away from the source data distribution is the accrual of the cost to ``travel'' to the best approximation ($\convergence$) and then to the source data distribution ($\bias$).

    \subsection{Imperfect Multitask Learning When Distributions Shift}\label{sec:indist}
        Finally, we now extend the result in \Cref{lem:imperfect-no-distshift} to the setting of imperfect learning with distribution shift.
        Recall from \Cref{sec:formulation} that we say distribution shift occurs when the target task distribution differs from the source task distribution ($\hypertargetdist \neq \hypersourcedist$).

        The extent of distribution shift captures the distance between the source and target task distributions.
        \begin{definition}[Extent of distribution shift ($\distshift$)]\label{def:distshift}
            The extent of distribution shift is
            $$\distshift \coloneqq \tv{\bary{\hypersourcedist}}{\bary{\hypertargetdist}}.$$
        \end{definition}

        \begin{remark}\label{rem:distshift}
            Notice that \Cref{def:distshift} accommodates any form of shift in the distribution of data the learner aims to predict.
            Since the barycenter of a task distribution $\bary{\hyperdist}$ fully specifies a first-order data distribution, the barycenter of the task distribution shifts (i.e., $\distshift > 0$) if and only if the first-order data distribution shifts from training to testing.

            As an example, consider the scenario of covariate shift, a shift in the input distribution in a supervised learning context.
            Even in the presence of covariate shift, the conditional distribution of outputs given inputs is unchanged, and so $\distshift = 0$.
            Closely related measures have been leveraged to measure the extent of covariate shift in the context of binary classification \citep{ben-david_analysis_2006,ben-david_theory_2010}.
        \end{remark}

        \Cref{thm:imperfect-distshift} provides an epistemic error bound for the \textit{imperfect learner in the presence of distribution shift}.
        \begin{theorem}[Epistemic error depends on task variability, model restrictions, data scarcity, and distribution shift]\label{thm:imperfect-distshift}
            Given a predictor $\preddist$, a source task distribution $\hypersourcedist$, a target multitask distribution $\hypertargetdist$ satisfying \Cref{as:closed}, and any $\margin \in \Rplus$,
            $$\trueprob{ \epistemicerr \geq \margin + \bias + \convergence + \distshift } \leq \frac{\maxvar{\hypertargetdist}}{\margin^2}.$$
        \end{theorem}
        \Cref{thm:imperfect-distshift} shows the consequence of moving the learner's ``destination'' away from the source data distribution: Now, it accrues the cost of ``travel'' from the source to target data distribution ($\distshift$).
        \Cref{cor:not-closed} in \Cref{ap:imperfect-distshift} provides a slightly modified version of \Cref{thm:imperfect-distshift} that does not require \Cref{as:closed}.

        \paragraph{Empirical illustration}
            To illustrate the implications of \Cref{thm:imperfect-distshift}, we constructed a multitask learning setting using the Iris data set from the UCI Machine Learning Repository \citep{iris_53}.
            In this setting, an imperfect learner encountered several target tasks that differed systematically from the source data on which it based its predictions.
            Details of the learning setting are provided in \Cref{ap:iris}.
            \begin{figure}[t!]
                \includegraphics[width=\linewidth]{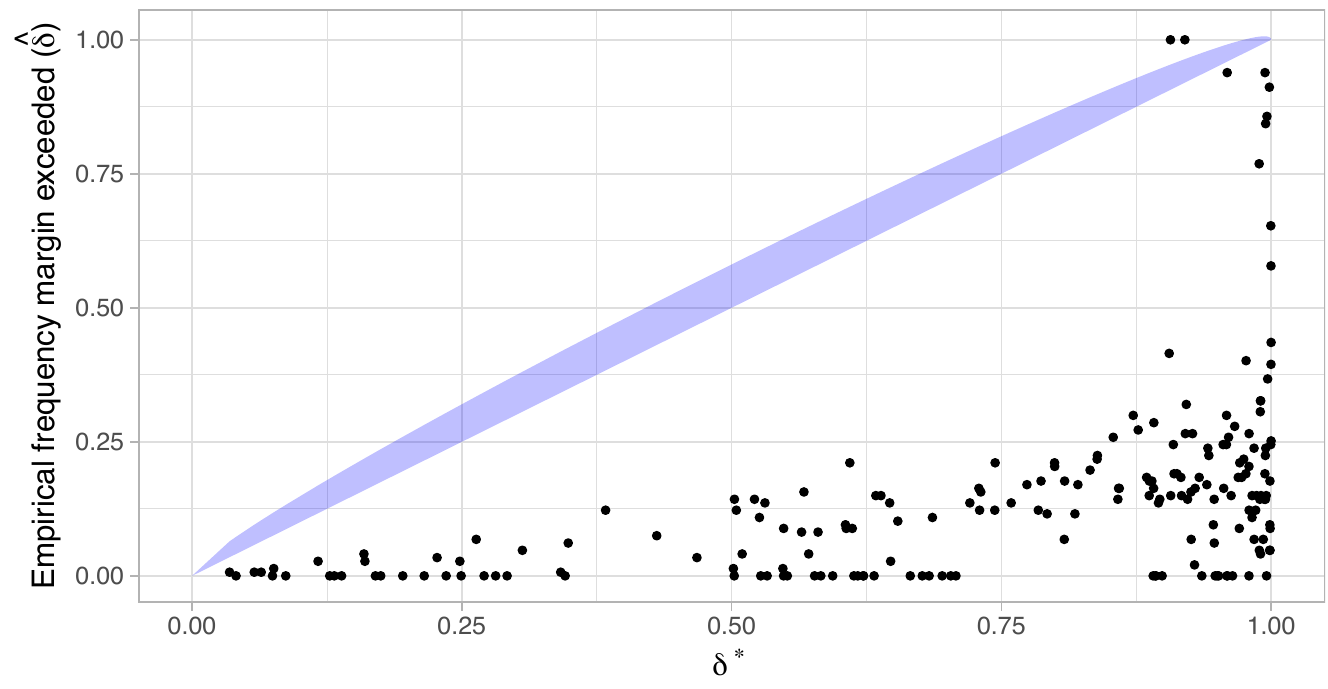}
                \caption{Results in a setting constructed using the Iris data set \citep{iris_53} to demonstrate the conclusions of \Cref{thm:imperfect-distshift}.
                    $x$-axis: $\delta^{\star} = \maxvar{\hypertargetdist} / \alpha^2$.
                    $y$-axis: $\hat{\delta}$ is the proportion of target tasks for which the epistemic error margin was exceeded.
                    The shaded region denotes approximately two standard errors above the value of $\delta^{\star}$ (computed according to a normal approximation).
                }
                \label{fig:iris-thm1-main}
            \end{figure}
            
            \Cref{fig:iris-thm1-main} shows that the empirical probability of exceeding the epistemic error margin almost never exceeds the theoretical value provided in \Cref{thm:imperfect-distshift}.\footnote{Note that where the bound appears to be exceeded at $\delta^{\star} \approx 1$, the normal approximation to the 95\% confidence region is especially inaccurate.}
            It also shows that the bound is in some cases considerably loose: The empirical probability of exceeding the epistemic error margin is often much smaller than the theoretical probability.
            This may indicate that the epistemic error margin can be much larger than the actual epistemic error the learner incurs.
            \Cref{sec:negtransfer} discusses a conceptual relationship between the amount of looseness in the epistemic error margin and the phenomenon of negative transfer.

        \paragraph{Incorporating aleatoric uncertainty}
            As discussed in \Cref{sec:related-work}, an epistemic error bound isolates the degree of reducible uncertainty, while a generalization bound captures error that can arise from epistemic (reducible) or aleatoric (irreducible) factors.

            \Cref{cor:cross-ent} extends \Cref{thm:imperfect-distshift} to provide a generalization bound on the learner's cross-entropy loss, a common measure of generalization performance.
            \Cref{cor:cross-ent} illustrates the contribution of aleatoric uncertainty to generalization error.
            The proof leverages a reverse Pinsker inequality \citep{sason_upper_2015} and is given in \Cref{sec:distance}.

            \begin{corollary}[Generalization error depends on aleatoric factors]\label{cor:cross-ent}
                Given that $\lvert \datasamples \rvert \in \mathbb{N}$ (\textbf{finite sample space}), a predictor $\preddist$ with p.m.f. $\widehat{p} ~ : ~ \datasamples \mapsto \left[ \underline{v}, \overline{v} \right]$, $\underline{v} > 0$, a source task distribution $\hypersourcedist$, a target multitask distribution $\hypertargetdist$, and any $\margin \in \Rplus$,
                $$\trueprob{ \loss \left( \preddist, \targettaskdist \right) \geq \frac{\left( \margin + \bias + \convergence + \distshift \right)^2}{\underline{v} / 2} + \mathcal{E} } \leq \frac{\maxvar{\hypertargetdist}}{\margin^2} $$
                where the cross-entropy loss function $\loss$ and entropy of the target task $\mathcal{E}$ are defined in \Cref{sec:distance}.
            \end{corollary}

            The entropy of the target task distribution $\mathcal{E}$ that appears in \Cref{cor:cross-ent} captures the dispersion of the target task distribution, i.e., of aleatoric uncertainty.
            The margin of generalization error that appears in \Cref{cor:cross-ent} decomposes generalization error into epistemic error (represented by the terms that also appear in the statement of \Cref{thm:imperfect-distshift}) and error due to aleatoric factors (represented by $\mathcal{E}$).

        \paragraph{Distribution shift from the learner's perspective}
            \Cref{def:distshift} provides a definition of distribution shift that is agnostic to the learner's chosen model class.
            However, some work conceptualizes distribution shift in a way that depends on the learner's chosen model class \citep{baxter_model_2000,redko_survey_2019}.
            This conceptualization is related to the concept of an inductive bias \citep{baxter_model_2000}: Specifying a model on the basis of an appropriate inductive bias (such as reduced hypothesis class complexity) can enhance robustness to distribution shift.

            \Cref{def:distshift-model} provides an alternative definition of the extent of distribution shift that captures what the learner perceives.
            Recall from \Cref{def:bias} that $\bestfitdist{S}$ is the best approximation to the source data within the learner's model class.
            \Cref{def:distshift-model} captures how much better $\bestfitdist{S}$ approximates the source data than the target data.
            It can be thought of as the {\em degree to which the learner's inductive biases induce robustness to the distribution shifts it encounters}.\footnote{Notice that unlike $\distshift$, $\distshiftLearner$ can be negative, implying that the learner's inductive biases can lead to a model that performs \textit{better} under distribution shift (see related discussion in \citet{tripuraneni_overparameterization_2021} on the strength of a distribution shift).}
            \begin{definition}[Extent of distribution shift: learner's view ($\distshiftLearner$)]\label{def:distshift-model}
                The extent of distribution shift the learner perceives is
                $$\distshiftLearner \coloneqq \tv{\bestfitdist{S}}{\bary{\hypertargetdist}} - \bias.$$
            \end{definition}

            \Cref{cor:excess-bias} shows that incorporating this alternative definition of distribution shift yields a similar decomposition to that given in \Cref{thm:imperfect-distshift}.
            \begin{theorem}[Epistemic error bound from the learner's perspective on distribution shift]\label{cor:excess-bias}
                Given a predictor $\preddist$, a source task distribution $\hypersourcedist$, a target multitask distribution $\hypertargetdist$ satisfying \Cref{as:closed}, and any $\margin \in \Rplus$,
                $$\trueprob{ \epistemicerr \geq \margin + \bias + \convergence + \distshiftLearner } \leq \frac{\maxvar{\hypertargetdist}}{\margin^2}.$$
            \end{theorem}
            As we show in \Cref{ap:imperfect-distshift}, \Cref{cor:excess-bias} provides a tighter epistemic error margin than \Cref{thm:imperfect-distshift}.\footnote{
                This is reminiscent of results from domain adaptation theory in which incorporating terms that reflect a learner's ``capacity to adapt'' generally leads to tighter inequalities \citep{redko_survey_2019}.
            }
            It reveals that the learner faces a trade-off (akin to, but distinct from, a bias--variance trade-off): To reduce epistemic error, the learner's predictor must fit but not overfit to the source data (small $\bias$ and $\distshiftLearner$).

        \paragraph{Computability of $\distshift$}
            So far, we have treated the components of the epistemic error margin as theoretical quantities.
            In some cases, one may want to directly approximate the terms in the bound to, for instance, obtain a concrete guarantee on a model's performance.
            In general, the learner does not have access to $\bestfitdist{}$ (to compute $\bias$), $\bary{\hypersourcedist}$ (to compute $\convergence$), or $\bary{\hypertargetdist}$ (to compute $\distshift$).
            If it did, it would not only be able to compute the epistemic error margin exactly, but have solved the learning problem it confronts in the first place.
            The learner does, however, have the source data, i.e., a set of finite samples from $\bary{\hypersourcedist}$.
            If the learner also has access to finite samples from the target tasks (from fine-tuning or later deployment in these tasks), this can facilitate a (perhaps retrospective) numerical approximation of the extent of distribution shift $\distshift$.

            An additional challenge in computing the terms in the bound is that the TV distance is generally intractable.
            As we discuss in \Cref{sec:distance}, the TV distance has close connections with other distance measures which can be more reliably estimated numerically.
            For example, Pinsker's inequality can be used to approximate $\distshift$ as a function of empirical estimates of the Kullback-Leibler divergence from $\bary{\hypersourcedist}$ to $\bary{\hypertargetdist}$ (we leverage this to estimate the TV distance in our synthetic empirical examples; see \Cref{ap:experiments}).\footnote{As another example, an empirical estimate of the $L_1$ distance $\ell_{L_1}$ between $\bary{\hypersourcedist}$ and $\bary{\hypertargetdist}$ provides an estimate of $\distshift$ by leveraging the fact that $\ell_{L_1}\left( \bary{\hypersourcedist}, \bary{\hypertargetdist} \right) = 2\distshift$.}

    \subsection{Negative Transfer}\label{sec:negtransfer}
        Negative transfer refers to the empirical phenomenon that the learner predicts worse in the target task after learning from the source data than before \citep{wang_characterizing_2019,zhang_survey_2023,sloman_proxy_2025}.
        This has a natural interpretation with respect to the quantities in \Cref{thm:imperfect-distshift}: We say \textbf{negative transfer} occurs when predictors closer to the source data distribution $\bary{\hypersourcedist}$ (i.e., which better characterize the source data distribution) commit larger epistemic errors than those further away.
    
        Informally, the provided bound is especially loose when it is much more direct to travel from $\preddist$ to $\targettask$ than to take a circuitous route via $\bary{\hypersourcedist}$.
        This is exactly the setting of negative transfer: The learner is better off applying a ``na\"{i}ve'' predictor when encountering $\targettask$ than they are adjusting their predictor according to the data accrued to $\bary{\hypersourcedist}$.
        
        This observation cautions of a conceptual slippery slope:
        In the presence of positive transfer, our results equip the learner with a framework to reason about how to reduce their epistemic error.
        However, the learners in most need of reducing their epistemic error (who experience negative transfer) are the least equipped to reason about how to do so (whose epistemic error margin does not closely track the epistemic error they experience).
        
        \paragraph{Empirical illustration}
            \Cref{fig:experiments-negtransfer,fig:iris-negtransfer} in \Cref{ap:experiments} demonstrate empirically the relationship between looseness in the epistemic error margin and the learner's experience of negative transfer, in the context of a synthetic Bayesian transfer learning setting (see \Cref{sec:cases}) and the setting constructed using the Iris data (used to produce \Cref{fig:iris-thm1-main}), respectively.
            Details of the learning settings and results are provided in \Cref{ap:experiments}.

\section{SPECIAL CASES}\label{sec:cases}
    While \Cref{thm:imperfect-distshift} makes no assumptions on the learning procedure nor on the nature of the distribution shift, additional knowledge can facilitate expression of the sources of epistemic error in ways that reflect specific aspects of the setting at hand.
    We here demonstrate this in the context of two probabilistic learning frameworks, Bayesian transfer learning and total variation neighborhoods.
    In these settings, \Cref{thm:imperfect-distshift} provides insight into the role of convergence of the posterior distribution and neighborhood size, respectively.
    To the extent that the modeler can control these aspects (through, for instance, specification of a diffuse prior), our results also provide the modeler with a means to control the epistemic error margin.

    \paragraph{Bayesian transfer learning}\label{sec:learningbias}
        \newcommand{\prior}{P_0^{\thetavals}}
        \newcommand{\posterior}{P_1^{\thetavals}}
        \newcommand{\bestfitthetadist}{P_{\star}^{\thetavals}}
        \newcommand{\CTheta}{\convergence^{\thetavals}}
    
        is a paradigm in which the learner uses Bayesian inference to infer a parameter that is shared across the source and target tasks \citep{suder_bayesian_2023}.
        The Bayesian transfer learner specifies a model class parameterized by a parameter $\thetaval \in \thetavals$, representing their uncertainty about the best-fitting value $\thetaval$ as a distribution $\prior$.
        After viewing data, the learner forms their posterior $\posterior$ with p.d.f. $p_1$.
        The learner's predictor $\preddist$ has density
        \begin{align*}
            \preddistpdf(\datasample) &= \int_{\thetavals} p(\datasample \vert \thetaval) ~ p_1(\thetaval) ~ \text{d}\thetaval.
        \end{align*}

        The predictor that best approximates the source data distribution $\bestfitdist{S}$ is induced by the \textbf{best approximate parameter distribution} $\bestfitthetadist$.

        \Cref{cor:bayesian} provides an epistemic error bound for the Bayesian transfer learner.
        \begin{corollary}[Bayesian transfer learning]\label{cor:bayesian}
            Given a posterior $\posterior$, a source task distribution $\hypersourcedist$, a target multitask distribution $\hypertargetdist$ satisfying \Cref{as:closed}, and any $\margin \in \Rplus$,
            $$\trueprob{ \epistemicerr \geq \margin + \bias + \CTheta + \distshift } \leq \frac{\maxvar{\hypertargetdist}}{\margin^2}$$
            where $\CTheta \coloneqq \tv{\posterior}{\bestfitthetadist}$.
        \end{corollary}
        \Cref{cor:bayesian} shows that the epistemic error margin depends on the degree to which the learner's posterior distribution resembles the best approximate parameter distribution: Posterior convergence reduces epistemic error.

        \vspace{2mm}
        \noindent \textit{Empirical illustration.}
            \Cref{fig:convergence} in \Cref{ap:btl-tv} shows the empirical relationship between posterior convergence and epistemic error in a synthetic Bayesian linear regression setting.
            Here, the posterior probability that the model assigns to the source data-generating coefficients (which captures the similarity between $\posterior$ and $\bestfitthetadist$) is negatively correlated with the epistemic error the model incurs.

    \paragraph{Total variation neighborhoods}
        model distribution shift, bounding the distance of the target task from the source task(s).
        They appear in contexts that range from the distributional robustness \citep{blanchet_distributionally_2024} to imprecise probabilities literature \citep{wasserman_bayes_1990}.

        \Cref{cor:eps} provides an epistemic error bound for the learner encountering shifts confined to a total variation neighborhood, defined in \Cref{as:eps}.
        \begin{assumption}\label{as:eps}
            There exists an $\eps \in (0,1)$ such that, for all $Q^t \in \mathrm{supp}({\hypertargetdist})$, we have that $\tv{\targettaskdist}{\task^s} \leq \eps$ for some $Q^s \in \mathrm{supp}({\hypersourcedist})$.
        \end{assumption}
        A bound under an alternative model, which considers the distance between the source and target task distributions, is provided in \Cref{ap:cases}.

        \begin{corollary}\label{cor:eps}
            Given a predictor $\preddist$, a target multitask distribution $\hypertargetdist$ satisfying \Cref{as:closed} and \Cref{as:closed-var} in \Cref{ap:tv} with p.d.f. $\hypertargetpdf ~ : ~ \distvals^T \mapsto \left[\minp^T,\maxp^T\right]$, a source task distribution $\hypersourcedist$ with p.d.f. $\hypersourcepdf ~ : ~ \distvals^T \mapsto \left[\minp^S,\maxp^S\right]$, $\minp^S > 0$, an $\eps \in (0,1)$ satisfying \Cref{as:eps}, and any $\margin \in \Rplus$,
            $$\trueprob{\epistemicerr \geq \margin + \bias + \convergence + \distshift} \leq \frac{\beta}{\margin^2} \left( {\mathbb{V}_{\overline{\dataval}} \left[ \hypersourcedist \right] + \mathrm{vol}\left( \hypertargetdist \right)} \right)$$
            where $\beta = \maxp^T/\minp^S$, $\mathrm{vol}\left( \hypertargetdist \right) \coloneqq \left( \mathrm{diam}\left( \hypersourcedist \right) + \eps \right)^2$, and $\overline{\dataval} \in \datavals$ and the diameter of the source task distribution $\mathrm{diam}\left( \hypersourcedist \right)$ are defined in \Cref{ap:tv}.
        \end{corollary}
        Intuitively, the learner's epistemic error increases with the scope and variability of the source tasks and with the size of the total variation neighborhood $\epsilon$.
        
        Notice that the bound in \Cref{cor:eps} becomes increasingly vacuous as $\maxp^T$ increases or $\minp^S$ decreases.
        For a task distribution $\hyperdist$ with p.d.f. $\hyperdistpdf ~ : ~ \distvals^T \mapsto \left[ \minp, \maxp \right]$, the size of the codomain $\left[ \minp, \maxp \right]$ is inherently connected to the (epistemic) uncertainty about the task the learner confronting $\hyperdist$ will encounter.
        A large (small) codomain indicates relative (un)certainty of encountering, or of not encountering, certain tasks.
        \Cref{cor:eps} is most informative in instances of relative task uncertainty.

    \vspace{2mm}
    \noindent \textit{Empirical illustration.}
        \Cref{fig:epsilon,fig:epsilon-delta} in \Cref{ap:btl-tv} show the empirical relationship between $\epsilon$ and epistemic error when the Bayesian learner encounters synthetic data arising from the setting described in \Cref{as:eps}.
        Here, the learner incurs higher epistemic error when $\epsilon$ is larger (i.e., target tasks are sampled from a larger total variation neighborhood around source tasks).

    \vspace{2mm}
    \noindent An epistemic error bound for the setting of Bayesian transfer learning when distributions shift in total variation neighborhoods is immediate from \Cref{cor:bayesian,cor:eps}, and provided in \Cref{ap:blt-tv-theory}.

\section{DISCUSSION}\label{sec:discussion}
    Unlike a generalization bound, an epistemic error bound provides a guarantee of a learner's amount of \textit{reducible} error.
    We provided a general decompositional epistemic error bound for the challenging setting of imperfect multitask learning in the presence of distribution shift.

    \paragraph{Limitations}
        While we intended the setting developed in \Cref{sec:formulation} to be maximally general, some of the assumptions may nevertheless be violated in real-world settings.
        Future work could investigate the development of epistemic error bounds in an even more general setting, allowing, for instance, for time dependence in the realization of tasks and data.
        Future work should also look towards the validation of our findings in higher-dimensional settings.

\section*{Acknowledgments}
    The authors thank Yusuf Sale for helpful discussion.
    This work was supported by EU grant (101120237 and ERC ODD-ML 101201120) and the Research Council of Finland Flagship programme: Finnish Center for Artificial Intelligence FCAI and decisions 359207, 359567, 358958.
    SJS and SK were supported by the UKRI Turing AI World-Leading Researcher Fellowship, [EP/W002973/1].

\bibliographystyle{apalike}
\bibliography{bibliography}

%% file: appendix.tex
\begin{center}
	{\LARGE \textbf{Appendix}}
\end{center}

The appendix is organized as follows:
\begin{itemize}
    \item \Cref{ap:related-work} discusses related works in more detail.
    \item \Cref{ap:proofs} provides complete proofs of all the mathematical results in the paper.
    \item \Cref{sec:distance} states and proves corollaries of \Cref{thm:imperfect-distshift} that provide epistemic error bounds for definitions of epistemic error in terms of cross-entropy, L1 distance, and Hellinger distance.
    \item \Cref{ap:experiments} provides the results and details of empirical demonstrations of our theoretical results.
\end{itemize}

\section{ADDITIONAL RELATED WORKS}\label{ap:related-work}
    Below we briefly describe statistical learning theoretic paradigms that deal with settings of multitask learning or distribution shift.
    The bound we provide in \Cref{thm:imperfect-distshift} applies in a setting more general than those studied under these paradigms.

    \paragraph{Multitask domain generalization} provides generalization bounds for the multitask setting \citep{baxter_model_2000,maurer_bounds_2006,maurer_sparse_2013,maurer_excess_2013,liu_algorithm_2017,deshmukh_generalization_2019}.
    The settings in which these bounds are derived do not account for distribution shift, and many assume the learner adopts a specific learning procedure \citep{maurer_bounds_2006,maurer_excess_2013,maurer_sparse_2013,liu_algorithm_2017,deshmukh_generalization_2019}.
    
    \paragraph{Domain adaptation theory} provides generalization bounds that apply in certain cases of distribution shift \citep{redko_survey_2019}.
    Unlike in our setting, domain adaptive learners are aware of the distribution shift and usually have available sufficient data points from the target task to adapt their choice of predictor.
    
    \paragraph{Credal learning theory} is a recent extension of statistical learning theory to credal sets \citep{caprio_credal_2024}, a model of second-order uncertainty \citep{walley_statistical_1991,caprio_constriction_2023}.
    Credal sets have been used to model multitask learning \citep{singh_domain_2024,lu_ibcl_2024} and distribution shift \citep{caprio_cbdl_2024}.
    While \citet{caprio_credal_2024} consider specifically the empirical risk minimizer, our bound is agnostic to both the type of learner (allowing for imperfect learners who do not minimize empirical risk) and model of distribution shift.

\subsection{Connection to Asymptotic Statistical Theory}\label{ap:ast}
    Like statistical learning theory, asymptotic statistical theory (AST) is a paradigm for assessing the performance of a learner.
    In parametric applications, prediction is done indirectly by first selecting the value of a set of parameters that most resembles the training data.
    Since they characterize generalization performance directly, PAC bounds apply equally to parametric and non-parametric learners.
    AST results, instead, characterize the limiting distribution of these parameter estimates.
    They usually take the form
    $$\lim_{N \rightarrow \infty} \mathrm{d} \left( \dist^{\thetavals}_N, \overrightarrow{\dist}^{\thetavals} \right) = 0$$
    for an appropriate distance function $\mathrm{d}$, the learner's chosen parameter distribution $\dist^{\thetavals}_N$ and an appropriate asymptotic distribution $\overrightarrow{\dist}^{\thetavals}$.\footnote{
        Subscript $N$ denotes the dependence of $\dist^{\thetavals}_N$ on sample size $N$.
        We use capital $N$ to denote the size of a sample of (usually independent and identically distributed) data points.
    } Each asymptotic statistical theoretic result is specific to a particular rule by which the learner selects $\dist^{\thetavals}_N$.
    For example, seminal results show that likelihood-based procedures --- where parameters are chosen according to their ability to make the observed data seem likely --- converge towards a symmetric distribution centered on a value that ``fits best'' in a well-defined way \citep{le-cam_asymptotic_1953,van-der-vaart_asymptotic_2000}.
    Unlike SLT results, which depend on the empirical risk, AST results are guarantees that can be made before data are seen.
    However, these guarantees hold only in the asymptotic regime and may not characterize behavior for any finite $N$.
    A subset of AST is concerned with characterizing \textit{rates of convergence} (e.g., \citealp{shen_rates_2001}).
    This literature considers the sequence of chosen parameter distributions implied by increasing values of $N$, characterizing how this sequence behaves as a function of $N$.
    
    Although our framework characterizes the closeness of predictive distributions, and so is fundamentally different from AST, we are inspired by several core AST concepts: (i) The probabilistic predictors we consider are in practice often parametric probability models in which epistemic uncertainty is represented as uncertainty about the parameters' value (see in particular our analysis of Bayesian transfer learning in \Cref{sec:cases,ap:cases}), (ii) our results are stated in terms of the distance between distributions, and (iii) a core quantity in our main result is the degree of convergence, which would allow the substitution of results on convergence in specific learning paradigms into our generic framework.

\section{PROOFS}\label{ap:proofs}
    We here provide complete proofs of all the mathematical results in the paper.
    
    \subsection{General Result}
        \subsubsection{Multitask Learning}\label{ap:proof-multitask}
            We here prove \Cref{cor:perfect-no-distshift}:

            \vspace{2mm}
            \noindent \textbf{\Cref{cor:perfect-no-distshift}} (Epistemic error depends on task variability)\textbf{.}\textit{
                Given a model class $\modelclass$, a predictor $\preddist \in \modelclass$, a source multitask distribution $\hypersourcedist$ satisfying \Cref{as:closed}, $\preddist = \bary{\hypersourcedist}$ (\textbf{perfect learning}), and $\hypertargetdist = \hypersourcedist$ (\textbf{no distribution shift}),
                $$\trueprob{ \epistemicerr \geq \margin } \leq \frac{\maxvar{\hypertargetdist}}{\margin^2}$$
                for any $\alpha \in \Rplus$.
            }

            \vspace{2mm}
            We first prove the following lemma, which is leveraged in the proof of \Cref{cor:perfect-no-distshift}:
            \begin{lemma}\label{lem:multitask}
                Given a multitask distribution $\hyperdist$, $\forall \taskdist \in \support{\hyperdist}$,
                \begin{align} \trueprob{ \tv{\bary{\hyperdist}}{Q} \geq \margin + \Delta } \leq \frac{\maxvar{\hyperdist}}{\margin^2} \label{eq:multitask-gen} \end{align}
                for any $\alpha \in \Rplus$ and any $\Delta \in \mathbb{R}^+$ (i.e., for arbitrarily small $\Delta$).
                If $\hyperdist$ satisfies \Cref{as:closed}, \eqref{eq:multitask-gen} additionally holds for $\Delta = 0$.
            \end{lemma}

            The proof of \Cref{lem:multitask} leverages the following proposition:
            \begin{proposition}\label{prop:sup-attain}
                Given a task distribution $\hyperdist$ and corresponding function $\tau_{\hyperdist}$ satisfying \Cref{as:closed}, $\forall \taskdist \in \support{\hyperdist}$, $\exists \dataval \in \datavals$ such that $\left| \taskdist\left( \dataval \right) - \bary{\hyperdist} \left( \dataval \right) \right| = \sup_{\dataval \in \datavals}{\tau_{\hyperdist}\left( \taskdist \right)}$.
            \end{proposition}
            \Cref{prop:sup-attain} says that the supremum of $\tau_{\hyperdist}\left( \taskdist \right)$ is always attained, and is a consequence of the fact that the supremum is always attained for closed sets of real numbers (\citealt{rudin_principles_1976}, Theorem 2.28).
            The closedness of $\tau_{\hyperdist}\left( \taskdist \right)$ is enforced by \Cref{as:closed}.

            We proceed with the proof of \Cref{lem:multitask}.
            
            \begin{proof}[Proof of \Cref{lem:multitask}]
                Given a multitask distribution $\hyperdist$ and any $b \in \Rplus$, we have, uniformly $\forall \dataval \in \datavals$, that
                \begin{align}
                    \trueprob{ \left| \taskdist(\dataval) - \bary{\hyperdist}(\dataval) \right| \geq b \wpnorm{\hyperdist} } &\leq \frac{1}{b^2} &\text{(Chebyshev's inequality; \Cref{def:var})} \label{eq:chebyshev} \\
                    \trueprob{ \left| \taskdist(\dataval) - \bary{\hyperdist}(\dataval) \right| \geq \margin } &\leq \frac{\var{\hyperdist}}{\margin^2} &\text{($\margin \equiv b \wpnorm{\hyperdist}$; \Cref{def:multitask})} \label{eq:alpha} \\
                    \trueprob{ \left| \taskdist(\dataval) - \bary{\hyperdist}(\dataval) \right| \geq \margin } &\leq \frac{\sup_{\dataval \in \datavals} \var{\hyperdist}}{\margin^2} & \nonumber \\
                    \trueprob{ \left| \taskdist(\dataval) - \bary{\hyperdist}(\dataval) \right| \geq \margin } &\leq \frac{\maxvar{\hyperdist}}{\margin^2} &\text{(\Cref{def:var})} \label{eq:tau}
                \end{align}
                where $\taskdist$ denotes a task $\taskdist \sim \hyperdist$.

                In line \eqref{eq:chebyshev}, Chebyshev's inequality requires that $\var{\hyperdist}$ is finite.
                This is ensured by the definition of $\var{\hyperdist}$ (see \Cref{rem:var}).

                Line \eqref{eq:alpha} defines $\alpha = f(b) \equiv b \sqrt{\var{\taskdist}}$.
                The positivity of $\sqrt{\var{\hyperdist}} \in (0,1]$ is ensured by the requirement of a multitask distribution (\Cref{def:multitask}).
                Given any $\dataval \in \datavals$, $f$ can evaluate to any $\alpha \in \Rplus$ by suitably choosing $b \in \Rplus$.
                Line \eqref{eq:alpha} therefore holds for any $\alpha \in \Rplus$.
                
                It then follows that, for any given $\overline{\dataval} \in \datavals$,
                \begin{align}
                    \trueprob{ \left| \taskdist(\overline{\dataval}) - \bary{\hyperdist}(\overline{\dataval}) \right| \geq \margin } &\leq \frac{\maxvar{\hyperdist}}{\margin^2} &(\overline{\dataval} \in \datavals) \label{eq:xbar} \\
                    \trueprob{ \left| \taskdist(\overline{\dataval}) - \bary{\hyperdist}(\overline{\dataval}) \right| + \Delta_{\overline{x}} \geq \margin + \Delta_{\overline{x}} } &\leq \frac{\maxvar{\hyperdist}}{\margin^2} &(\Delta_{\overline{x}} \equiv \sup_{\dataval \in \datavals} \left| \taskdist(\dataval) - \bary{\hyperdist}(\dataval) \right| - \left| \taskdist(\overline{\dataval}) - \bary{\hyperdist}(\overline{\dataval}) \right|) \nonumber \\
                    \trueprob{  \sup_{\dataval \in \datavals} \left| \taskdist(\dataval) - \bary{\hyperdist}(\dataval) \right| \geq \margin + \Delta_{\overline{x}} } &\leq \frac{\maxvar{\hyperdist}}{\margin^2} \nonumber \\
                    \trueprob{  \tv{\bary{\hyperdist}}{Q} \geq \margin + \Delta_{\overline{x}} } &\leq \frac{\maxvar{\hyperdist}}{\margin^2}. &\text{(\Cref{def:tv})} \nonumber
                \end{align}
                If $\hyperdist$ satisfies \Cref{as:closed}, \Cref{prop:sup-attain} tells us that the supremum of $\tvset\left( \taskdist \right) \coloneqq \left\{ \left| \taskdist\left( \dataval \right) - \bary{\hyperdist}\left( \dataval \right) \right| : ~ \dataval \in \datavals \right\}$ is attained.
                For any $\overline{\dataval} \in \datavals$ that achieves the supremum, $\Delta_{\overline{\dataval}} = 0$ and the statement \eqref{eq:multitask-gen} holds for $\Delta = 0$.
                
                If $\hyperdist$ does not satisfy \Cref{as:closed}, $\Delta_{\overline{x}}$ can be made arbitrarily small by suitably choosing $\overline{\dataval} \in \datavals$, i.e., the statement \eqref{eq:multitask-gen} holds for any $\Delta \in \mathbb{R}^+$.
            \end{proof}

            We proceed with the proof of \Cref{cor:perfect-no-distshift}.
            \begin{proof}[Proof of \Cref{cor:perfect-no-distshift}]
                Given a source multitask distribution $\hypersourcedist$ satisfying \Cref{as:closed} and any $\alpha \in \Rplus$, we have that
                \begin{align}
                    \trueprob{ \tv{\bary{\hypersourcedist}}{\sourcetaskdist} \geq \margin } &\leq \frac{\maxvar{\hypersourcedist}}{\margin^2} &\text{(\Cref{lem:multitask})} \nonumber \\
                    \trueprob{ \tv{\preddist}{\sourcetaskdist} \geq \margin } &\leq \frac{\maxvar{\hypersourcedist}}{\margin^2} &\text{($\preddist = \bary{\hypersourcedist}$)} \nonumber \\
                    \trueprob{ \tv{\preddist}{\targettaskdist} \geq \margin } &\leq \frac{\maxvar{\hypertargetdist}}{\margin^2} &\text{($\hypertargetdist = \hypersourcedist$)} \nonumber \\
                    \trueprob{ \epistemicerr \geq \margin } &\leq \frac{\maxvar{\hypertargetdist}}{\margin^2} &\text{(\Cref{def:epistemic})} \nonumber
                \end{align}
                where $\sourcetaskdist$ denotes a source task $\sourcetaskdist \sim \hypersourcedist$.
            \end{proof}

        \subsubsection{Imperfect Multitask Learning}
            We here prove \Cref{lem:imperfect-no-distshift}:

            \vspace{2mm}
            \noindent\textbf{\Cref{lem:imperfect-no-distshift}} (Epistemic error depends on task variability, model restrictions, and data scarcity)\textbf{.}\textit{
                Given a predictor $\preddist$, a source multitask distribution $\hypersourcedist$ satisfying \Cref{as:closed}, and $\hypertargetdist = \hypersourcedist$ (\textbf{no distribution shift}),
                $$\trueprob{ \epistemicerr \geq \margin + \bias + \convergence } \leq \frac{\maxvar{\hypertargetdist}}{\margin^2}$$
                for any $\alpha \in \Rplus$.
            }
            \vspace{2mm}

            The proof of \Cref{lem:imperfect-no-distshift} leverages the following lemma:
            \begin{lemma}\label{lem:bc}
                Given a predictor $\preddist$ and a source task distribution $\hypersourcedist$,
                $$\tv{\preddist}{\bary{\hypersourcedist}} \leq \bias + \convergence.$$
            \end{lemma}
            \Cref{lem:bc} follows directly from \Cref{def:bias,def:convergence} and an application of the triangle inequality. 

            We proceed with the proof of \Cref{lem:imperfect-no-distshift}.
            \begin{proof}[Proof of \Cref{lem:imperfect-no-distshift}]
                Given a predictor $\preddist$, a source multitask distribution $\hypersourcedist$ satisfying \Cref{as:closed}, and any $\alpha \in \Rplus$, we have that
                \begin{align}
                    \trueprob{ \tv{\bary{\hypersourcedist}}{\sourcetaskdist} \geq \margin } &\leq \frac{\maxvar{\hypersourcedist}}{\margin^2} &\text{(\Cref{lem:multitask})} \nonumber \\
                    \trueprob{ \tv{\preddist}{\bary{\hypersourcedist}} + \tv{\bary{\hypersourcedist}}{\sourcetask} \geq \margin + \bias + \convergence } &\leq \frac{\maxvar{\hypersourcedist}}{\margin^2} &\text{(\Cref{lem:bc})} \nonumber \\
                    \trueprob{ \tv{\preddist}{\sourcetaskdist} \geq \margin + \bias + \convergence } &\leq \frac{\maxvar{\hypersourcedist}}{\margin^2} &\text{(Triangle inequality)} \nonumber \\
                    \trueprob{ \tv{\preddist}{\targettaskdist} \geq \margin + \bias + \convergence } &\leq \frac{\maxvar{\hypertargetdist}}{\margin^2} &\text{($\hypertargetdist = \hypersourcedist$)} \nonumber \\
                    \trueprob{ \epistemicerr \geq \margin + \bias + \convergence } &\leq \frac{\maxvar{\hypertargetdist}}{\margin^2} &\text{(\Cref{def:epistemic})} \nonumber
                \end{align}
                where $\sourcetaskdist$ denotes a source task $\sourcetaskdist \sim \hypersourcedist$.
            \end{proof}

        \subsubsection{Imperfect Multitask Learning When Distributions Shift}\label{ap:imperfect-distshift}
            We here prove \Cref{thm:imperfect-distshift}:

            \vspace{2mm}
            \noindent\textbf{\Cref{thm:imperfect-distshift}} (Epistemic error depends on task variability, model restrictions, data scarcity, and distribution shift)\textbf{.}\textit{
                Given a predictor $\preddist$, a source task distribution $\hypersourcedist$, and a target multitask distribution $\hypertargetdist$ satisfying \Cref{as:closed},
                $$\trueprob{ \epistemicerr \geq \margin + \bias + \convergence + \distshift } \leq \frac{\maxvar{\hypertargetdist}}{\margin^2}$$
                for any $\alpha \in \Rplus$.
            }
        
            \begin{proof}[Proof of \Cref{thm:imperfect-distshift}]
                Given a predictor $\preddist$, a source task distribution $\hypersourcedist$, a target multitask distribution $\hypertargetdist$ satisfying \Cref{as:closed}, and any $\alpha \in \Rplus$, we have that
                \begin{align}
                    \trueprob{ \tv{\bary{\hypertargetdist}}{\targettaskdist} \geq \margin } &\leq \frac{\maxvar{\hypertargetdist}}{\margin^2} &\text{(\Cref{lem:multitask})} \nonumber \\
                    \trueprob{ \tv{\bary{\hypersourcedist}}{\bary{\hypertargetdist}} + \tv{\bary{\hypertargetdist}}{\targettaskdist} \geq \margin + \distshift } &\leq \frac{\maxvar{\hypertargetdist}}{\margin^2} &\text{(\Cref{def:distshift})} \nonumber \\
                    \trueprob{ \tv{\bary{\hypersourcedist}}{\targettaskdist} \geq \margin + \distshift } &\leq \frac{\maxvar{\hypertargetdist}}{\margin^2} &\text{(Triangle inequality)} \nonumber \\
                    \trueprob{ \tv{\preddist}{\bary{\hypersourcedist}} + \tv{\bary{\hypersourcedist}}{\targettaskdist} \geq \margin + \tv{\preddist}{\bary{\hypersourcedist}} + \distshift } &\leq \frac{\maxvar{\hypertargetdist}}{\margin^2} \nonumber \\
                    \trueprob{ \tv{\preddist}{\bary{\hypersourcedist}} + \tv{\bary{\hypersourcedist}}{\targettaskdist} \geq \margin + \bias + \convergence + \distshift } &\leq \frac{\maxvar{\hypertargetdist}}{\margin^2} &\text{(\Cref{lem:bc})} \nonumber \\
                    \trueprob{ \tv{\preddist}{\targettaskdist} \geq \margin + \bias + \convergence + \distshift } &\leq \frac{\maxvar{\hypertargetdist}}{\margin^2} &\text{(Triangle inequality)} \nonumber \\
                    \trueprob{ \epistemicerr \geq \margin + \bias + \convergence + \distshift } &\leq \frac{\maxvar{\hypertargetdist}}{\margin^2} &\text{(\Cref{def:epistemic})} \nonumber
                \end{align}
                where $\targettaskdist$ denotes a target task $\targettaskdist \sim \hypertargetdist$.
            \end{proof}

            \paragraph{Violation of \Cref{as:closed}}
                The following corollary to \Cref{thm:imperfect-distshift} provides a general epistemic error bound in case \Cref{as:closed} does not hold:
                \begin{corollary}[Modification to \Cref{thm:imperfect-distshift} if \Cref{as:closed} is violated]\label{cor:not-closed}
                    Given a predictor $\preddist$, a source task distribution $\hypersourcedist$, and a target multitask distribution $\hypertargetdist$,
                    $$\trueprob{ \epistemicerr \geq \margin + \bias + \convergence + \distshift + \Delta } \leq \frac{\maxvar{\hypertargetdist}}{\margin^2}$$
                    for any $\margin \in \Rplus$ and any $\Delta \in \mathbb{R}^+$ (i.e., for arbitrarily small $\Delta$).
                \end{corollary}

                \Cref{cor:not-closed} tells us that violation of \Cref{as:closed} results in an arbitrarily small inflation to the epistemic error margin and does not affect the conclusions drawn from \Cref{thm:imperfect-distshift} regarding the contributions of each source to epistemic error.
                
                For completeness, we here include the proof of \Cref{cor:not-closed}.
                The proof follows the same steps as the proof of \Cref{thm:imperfect-distshift}, incorporating the inflation to the epistemic error margin in each step.
                \begin{proof}[Proof of \Cref{cor:not-closed}]
                Given a predictor $\preddist$, a source task distribution $\hypersourcedist$, a target multitask distribution $\hypertargetdist$, any $\margin \in \Rplus$, and any $\Delta \in \mathbb{R}^+$, we have that
                \begin{align}
                    \trueprob{ \tv{\bary{\hypertargetdist}}{\targettaskdist} \geq \margin + \Delta } &\leq \frac{\maxvar{\hypertargetdist}}{\margin^2} &\text{(\Cref{lem:multitask})} \nonumber \\
                    \trueprob{ \tv{\bary{\hypersourcedist}}{\bary{\hypertargetdist}} + \tv{\bary{\hypertargetdist}}{\targettaskdist} \geq \margin + \distshift + \Delta } &\leq \frac{\maxvar{\hypertargetdist}}{\margin^2} &\text{(\Cref{def:distshift})} \nonumber \\
                    \trueprob{ \tv{\bary{\hypersourcedist}}{\targettaskdist} \geq \margin + \distshift + \Delta } &\leq \frac{\maxvar{\hypertargetdist}}{\margin^2} &\text{(Triangle inequality)} \nonumber \\
                    \trueprob{ \tv{\preddist}{\bary{\hypersourcedist}} + \tv{\bary{\hypersourcedist}}{\targettaskdist} \geq \margin + \tv{\preddist}{\bary{\hypersourcedist}} + \distshift + \Delta } &\leq \frac{\maxvar{\hypertargetdist}}{\margin^2} \nonumber \\
                    \trueprob{ \tv{\preddist}{\bary{\hypersourcedist}} + \tv{\bary{\hypersourcedist}}{\targettaskdist} \geq \margin + \bias + \convergence + \distshift + \Delta } &\leq \frac{\maxvar{\hypertargetdist}}{\margin^2} &\text{(\Cref{lem:bc})} \nonumber \\
                    \trueprob{ \tv{\preddist}{\targettaskdist} \geq \margin + \bias + \convergence + \distshift + \Delta } &\leq \frac{\maxvar{\hypertargetdist}}{\margin^2} &\text{(Triangle inequality)} \nonumber \\
                    \trueprob{ \epistemicerr \geq \margin + \bias + \convergence + \distshift + \Delta } &\leq \frac{\maxvar{\hypertargetdist}}{\margin^2} &\text{(\Cref{def:epistemic})} \nonumber
                \end{align}
                where $\targettaskdist$ denotes a target task $\targettaskdist \sim \hypertargetdist$.
            \end{proof}

            \paragraph{Distribution shift from the learner's perspective}
                We here prove \Cref{cor:excess-bias}:

                \vspace{2mm}
                \noindent\textbf{\Cref{cor:excess-bias}} (Epistemic error bound from the learner's perspective on distribution shift)\textbf{.}\textit{
                    Given a predictor $\preddist$, a source task distribution $\hypersourcedist$, and a target multitask distribution $\hypertargetdist$ satisfying \Cref{as:closed},
                    $$\trueprob{ \epistemicerr \geq \margin + \bias + \convergence + \distshiftLearner } \leq \frac{\maxvar{\hypertargetdist}}{\margin^2}$$
                    for any $\margin \in \Rplus$.
                }

                \vspace{2mm}
                The proof of \Cref{cor:excess-bias} leverages the following lemma:
                \begin{lemma}\label{lem:d-pred-baryT}
                    Given a predictor $\preddist$, a target task distribution $\hypersourcedist$, and a target task distribution $\hypertargetdist$,
                    $$\tv{\preddist}{\bary{\hypertargetdist}} \leq \bias + \convergence + \distshiftLearner.$$
                \end{lemma}
                \Cref{lem:d-pred-baryT} follows directly from \Cref{def:convergence,def:distshift-model} and an application of the triangle inequality.

                We proceed with the proof of \Cref{cor:excess-bias}:
                \begin{proof}[Proof of \Cref{cor:excess-bias}]
                    Given a predictor $\preddist$, a target task distribution $\hypersourcedist$, a target multitask distribution $\hypertargetdist$ satisfying \Cref{as:closed}, and any $\margin \in \Rplus$, we have that
                    \begin{align}
                        \trueprob{ \tv{\bary{\hypertargetdist}}{\targettaskdist} \geq \margin } &\leq \frac{\maxvar{\hypertargetdist}}{\margin^2} &\text{(\Cref{lem:multitask})} \nonumber \\
                        \trueprob{ \tv{\preddist}{\bary{\hypertargetdist}} + \tv{\bary{\hypertargetdist}}{\targettask} \geq \margin + \bias + \convergence + \distshiftLearner } &\leq \frac{\maxvar{\hypertargetdist}}{\margin^2} &\text{(\Cref{lem:d-pred-baryT})} \nonumber \\
                        \trueprob{ \tv{\preddist}{\targettaskdist} \geq \margin + \bias + \convergence + \distshiftLearner } &\leq \frac{\maxvar{\hypertargetdist}}{\margin^2} &\text{(Triangle inequality)} \nonumber
                    \end{align}
                    where $\targettaskdist$ denotes a target task $\targettaskdist \sim \hypertargetdist$.
                \end{proof}

                \Cref{prop:d-compare} expresses the claim made in \Cref{sec:indist} that \Cref{cor:excess-bias} provides a tighter epistemic error margin than \Cref{thm:imperfect-distshift}.
                \begin{proposition}[Distribution shift is less extreme from the learner's perspective]\label{prop:d-compare}
                    Given a predictor $\preddist$, a source task distribution $\hypersourcedist$, and a target task distribution $\hypertargetdist$,
                    $$\distshift \geq \distshiftLearner.$$
                \end{proposition}

                \begin{proof}[Proof of \Cref{prop:d-compare}]
                    Given a predictor $\preddist$, a source task distribution $\hypersourcedist$, and a target task distribution $\hypertargetdist$, we have that
                    \begin{align}
                        \tv{\bestfitdist{S}}{\bary{\hypersourcedist}} + \tv{\bary{\hypersourcedist}}{\bary{\hypertargetdist}} &\geq \tv{\bestfitdist{S}}{\bary{\hypertargetdist}} &\text{(Triangle inequality)} \nonumber \\
                        \tv{\bestfitdist{S}}{\bary{\hypersourcedist}} + \tv{\bary{\hypersourcedist}}{\bary{\hypertargetdist}} &\geq \tv{\bestfitdist{S}}{\bary{\hypersourcedist}} + \tv{\bestfitdist{S}}{\bary{\hypertargetdist}} - \tv{\bestfitdist{S}}{\bary{\hypersourcedist}} \nonumber \\
                        \bias + \distshift &\geq \bias + \distshiftLearner &\text{(\Cref{def:bias,def:distshift,def:distshift-model})} \nonumber \\
                        \distshift &\geq \distshiftLearner. \nonumber
                    \end{align}
                \end{proof}

    \subsection{Special Cases}\label{ap:cases}
        \subsubsection{Bayesian Transfer Learning}
            We here prove \Cref{cor:bayesian}:

            \vspace{2mm}
            \noindent \textbf{\Cref{cor:bayesian}} (Bayesian transfer learning)\textbf{.}\textit{
                Given a posterior $\posterior$, a source task distribution $\hypersourcedist$, and a target multitask distribution $\hypertargetdist$ satisfying \Cref{as:closed},
                $$\trueprob{ \epistemicerr \geq \margin + \bias + \CTheta + \distshift } \leq \frac{\maxvar{\hypertargetdist}}{\margin^2}$$
                for any $\margin \in \Rplus$ and where $\CTheta \coloneqq \tv{\posterior}{\bestfitthetadist}$.
            }

            \vspace{2mm}
            We first prove the following lemma, which is leveraged in the proof of \Cref{cor:bayesian}:
            \begin{lemma}\label{lem:convergence}
                Given a posterior $\posterior$ and a source task distribution $\hypersourcedist$,
                $$\convergence \leq \CTheta$$ for $\CTheta \coloneqq \tv{\posterior}{\bestfitthetadist}$.
            \end{lemma}
            
            \begin{proof}[Proof of \Cref{lem:convergence}]
                Given a posterior $\posterior$ and a source task distribution $\hypersourcedist$, we have that
                \begin{align}
                    \convergence &= \tv{\preddist}{\bestfitdist{S}} &\text{(\Cref{def:convergence})} \nonumber \\
                    &= \sup_{\dataval \in \datavals}{\left| \preddist(\dataval) - \bestfitdist{}(\dataval) \right|} &\text{(\Cref{def:tv})} \nonumber \\
                    &= \sup_{\dataval \in \datavals}{\left| \int_{\thetavals} \preddist(\dataval \vert \thetaval) ~ p_1(\thetaval) ~ d\thetaval - \int_{\thetavals} \preddist(\dataval \vert \thetaval) ~ p_{\star}(\thetaval) ~ d\thetaval \right|} \nonumber \\
                    &= \sup_{\dataval \in \datavals}{\left| \int_{\thetavals} \left( p_1(\thetaval) - p_{\star}(\thetaval) \right) \preddist(\dataval \vert \thetaval) ~ d\thetaval \right|} \nonumber \\
                    &\leq \left| \int_{\thetavals} \left( p_1(\thetaval) - p_{\star}(\thetaval) \right) ~ d\thetaval \right| &\text{($\forall \dataval \in \datavals, ~ 0 \leq \preddist(\dataval \vert \thetaval) \leq 1$)} \nonumber \\
                    &\leq \left| \int_{\thetavals^+} \left( p_1(\thetaval) - p_{\star}(\thetaval) \right) d\thetaval \right|
                    &\text{($\thetavals^+ \equiv \{ \thetaval \in \thetavals : p_1(\thetaval) \geq p_{\star}(\thetaval) \}$)} \nonumber \\
                    &= \left| \int_{\thetavals^+} p_1(\thetaval) ~ d\thetaval - \int_{\thetavals^+} p_{\star}(\thetaval) ~ d\thetaval \right| \nonumber \\
                    &= \left| \posterior(\thetavals^+) - \bestfitthetadist(\thetavals^+) \right| \nonumber \\
                    &\leq \tv{\posterior}{\bestfitthetadist} \nonumber
                \end{align}
                where $p_{\star}$ is the p.d.f. of the best approximate parameter distribution $\bestfitthetadist$.
            \end{proof}

            Direct substitution of the result from \Cref{lem:convergence} into the epistemic error margin in the statement of \Cref{thm:imperfect-distshift} completes the proof of \Cref{cor:bayesian}.

        \subsubsection{Total Variation Neighborhoods}\label{ap:tv}
            We here prove \Cref{cor:eps}, which depends on the following definition of the \textit{diameter} of a task distribution:
            \begin{definition}[Diameter of $\hyperdist$]\label{def:diam}
                Given a task distribution $\hyperdist$, the diameter of $\hyperdist$ is
                $$\mathrm{diam}\left( \hyperdist \right) \coloneqq \sup_{\taskdist, \repl{\taskdist} \in \support{\hyperdist} \times \support{\hyperdist}} \tv{\taskdist}{\repl{\taskdist}}.$$
            \end{definition}

            We also require the following assumption:
            \begin{assumption}\label{as:closed-var}
                A task distribution $\hyperdist$ satisfies this assumption if $\upsilon_{\hyperdist} \coloneqq \left\{ \var{\hyperdist} : ~ \dataval \in \datavals \right\}$ is a closed subset of the unit interval.
            \end{assumption}

            \noindent \textbf{\Cref{cor:eps}} (Total variation neighborhoods)\textbf{.}\textit{
                Given a predictor $\preddist$, a target multitask distribution $\hypertargetdist$ satisfying \Cref{as:closed,as:closed-var} with p.d.f. $\hypertargetpdf ~ : ~ \distvals^T \mapsto \left[\minp^T,\maxp^T\right]$, a source task distribution $\hypersourcedist$ with p.d.f. $\hypersourcepdf ~ : ~ \distvals^T \mapsto \left[\minp^S,\maxp^S\right]$, $\minp^S > 0$, and an $\eps \in (0,1)$ satisfying \Cref{as:eps},
                $$\trueprob{\epistemicerr \geq \margin + \bias + \convergence + \distshift} \leq \frac{\beta}{\margin^2} \left( {\mathbb{V}_{\overline{\dataval}} \left[ \hypersourcedist \right] + \mathrm{vol}\left( \hypertargetdist \right)} \right)$$
                for any $\margin \in \Rplus$ and where $\beta = \maxp^T/\minp^S$, $\overline{\dataval} \in \datavals$ such that $\mathbb{V}_{\overline{\dataval}} \left[ \hypertargetdist \right] = \maxvar{\hypertargetdist}$, and $\mathrm{vol}\left( \hypertargetdist \right) \coloneqq \left( \mathrm{diam}\left( \hypersourcedist \right) + \eps \right)^2$.
            }

            We first prove the following lemmas, which are leveraged in the proof of \Cref{cor:eps}:
            \begin{lemma}\label{lem:V}
                Given a target task distribution $\hypertargetdist$ satisfying \Cref{as:closed-var} with p.d.f. $\hypertargetpdf ~ : ~ \distvals^T \mapsto \left[\minp^T,\maxp^T\right]$ and a source task distribution $\hypersourcedist$ with p.d.f. $\hypersourcepdf ~ : ~ \distvals^T \mapsto \left[\minp^S,\maxp^S\right]$, $\minp^S > 0$,
                $$\maxvar{\hypertargetdist} \leq \beta \left( \mathbb{V}_{\overline{\dataval}} \left[ \hypersourcedist \right] + \distshift^2 \right)$$
                for $\beta = \maxp^T/\minp^S$ and $\overline{\dataval} \in \datavals$ such that $\mathbb{V}_{\overline{\dataval}} \left[ \hypertargetdist \right] = \maxvar{\hypertargetdist}$.
            \end{lemma}

            The proof of \Cref{lem:V} leverages the following proposition:
            \begin{proposition}\label{prop:maxvar-attain}
                Given a task distribution $\hyperdist$ and corresponding set $\upsilon_{\hyperdist}$ satisfying \Cref{as:closed-var}, $\exists \dataval \in \datavals$ such that $\var{\hyperdist} = \sup_{\dataval \in \datavals} \upsilon_{\hyperdist}$.
            \end{proposition}
            \Cref{prop:maxvar-attain} says that the supremum of $\upsilon_{\hyperdist}$ is always attained.
            Like \Cref{prop:sup-attain}, it is a consequence of the fact that the supremum is always attained for closed sets of real numbers (\citet{rudin_principles_1976}, Theorem 2.28).

            We proceed with the proof of \Cref{lem:V}.
            \begin{proof}[Proof of \Cref{lem:V}]
                Given a target task distribution $\hypertargetdist$ with p.d.f. $\hypertargetpdf ~ : ~ \distvals^T \mapsto \left[\minp^T,\maxp^T\right]$ and a source task distribution $\hypersourcedist$ with p.d.f. $\hypersourcepdf ~ : ~ \distvals^T \mapsto \left[\minp^S,\maxp^S\right]$, $\minp^S > 0$, we have, uniformly $\forall \dataval \in \datavals$, that
                \begin{align}
                    \var{\hypertargetdist} &= \int_{\distvals^T} \left( \taskdist(\dataval) - \bary{\hypertargetdist}(\dataval) \right)^2 ~ \hypertargetpdf(Q) ~ d Q &\text{(\Cref{def:var})} \nonumber \\
                    &= \int_{\distvals^T} \left( \left( \taskdist(\dataval) - \bary{\hypersourcedist}(\dataval) \right) + \left( \bary{\hypersourcedist}(\dataval) - \bary{\hypertargetdist}(\dataval) \right) \right)^2 ~ \left( \hypersourcepdf(Q) \frac{\hypertargetpdf(Q)}{\hypersourcepdf(Q)} \right) d Q \nonumber \\
                    &\leq \frac{\maxp^T}{\minp^S} \int_{\distvals^T} \left( \left( \taskdist(\dataval) - \bary{\hypersourcedist}(\dataval) \right)
                    + \left( \bary{\hypersourcedist}(\dataval) - \bary{\hypertargetdist}(\dataval) \right) \right)^2 \hypersourcepdf(Q) ~ dQ \nonumber \\
                    &= \frac{\maxp^T}{\minp^S} \int_{\distvals^T} \left( \left( \taskdist(\dataval) - \bary{\hypersourcedist}(\dataval) \right)^2
                    + 2 \left( \taskdist(\dataval) - \bary{\hypersourcedist}(\dataval) \right)
                    \left( \bary{\hypersourcedist}(\dataval) - \bary{\hypertargetdist}(\dataval) \right) \right. \nonumber \\
                    &\qquad\qquad\qquad\qquad\qquad
                    \left. + \left( \bary{\hypersourcedist}(\dataval) - \bary{\hypertargetdist}(\dataval) \right)^2 \right) \hypersourcepdf(Q) ~ dQ \nonumber \\
                    &= \frac{\maxp^T}{\minp^S} \Bigg( \var{\hypersourcedist} + 2 \left( \bary{\hypersourcedist}(\dataval) - \bary{\hypertargetdist}(\dataval) \right) \int_{\distvals^T} \left( \taskdist(\dataval) - \bary{\hypersourcedist}(\dataval) \right) \hypersourcepdf(Q) ~ dQ &\text{(\Cref{def:var})} \nonumber \\
                    &\qquad\qquad\qquad\qquad\qquad + \left( \bary{\hypersourcedist}(\dataval) - \bary{\hypertargetdist}(\dataval) \right)^2 \int_{\distvals^T} \hypersourcepdf(Q) ~ dQ \Bigg) \nonumber \\
                    &= \frac{\maxp^T}{\minp^S} \left( \var{\hypersourcedist} + \left( \bary{\hypersourcedist}(\dataval) - \bary{\hypertargetdist}(\dataval) \right)^2 \right) &\text{(\Cref{def:bary}; $\int_{\distvals^T}\hypersourcepdf(Q) ~ dQ = 1$)} \nonumber \\
                    &\leq \frac{\maxp^T}{\minp^S} \left( \var{\hypersourcedist} + \tv{\bary{\hypersourcedist}}{\bary{\hypertargetdist}}^2 \right) &\text{(\Cref{def:tv})} \nonumber \\
                    &= \frac{\maxp^T}{\minp^S} \left( \var{\hypersourcedist} + \distshift^2 \right). &\text{(\Cref{def:distshift})} \nonumber 
                \end{align}
                Now define $\overline{\datavals} \equiv \left\{ \dataval \in \datavals \mid \var{\hypertargetdist} = \maxvar{\hypertargetdist} \right\}$.
                Recalling \Cref{def:var}, $\overline{\datavals}$ contains the values of $\dataval$ at which the supremum of the variance of $\hypertargetdist$ is attained.
                If $\hypertargetdist$ satisfies \Cref{as:closed-var}, \Cref{prop:maxvar-attain} ensures that $\overline{\datavals} \neq \emptyset$, i.e., that the supremum is attained by at least one $\dataval \in \datavals$.
                It then follows that, for any given $\overline{\dataval} \in \overline{\datavals} \subseteq \datavals$,
                \begin{align}
                    \mathbb{V}_{\overline{\dataval}} \left[ \hypertargetdist \right] &\leq \frac{\maxp^T}{\minp^S} \left( \mathbb{V}_{\overline{\dataval}} \left[ \hypersourcedist \right] + \distshift^2 \right) & \nonumber \\
                    \maxvar{\hypertargetdist} &\leq \frac{\maxp^T}{\minp^S} \left( \mathbb{V}_{\overline{\dataval}} \left[ \hypersourcedist \right] + \distshift^2 \right). & \nonumber
                \end{align}
            \end{proof}

            \begin{lemma}\label{lem:diam-bary}
                Given a task distribution $\hyperdist$,
                $\forall \taskdist \in \support{\hyperdist}$,
                $$\tv{\taskdist}{\bary{\hyperdist}} \leq \mathrm{diam}\left( \hyperdist \right).$$
            \end{lemma}
            \begin{proof}[Proof of \Cref{lem:diam-bary}]
                Given a task distribution $\hyperdist$, we have, uniformly $\forall \taskdist \in \support{\hyperdist}$, that
                \begin{align}
                    \tv{\taskdist}{\bary{\hyperdist}} &= \sup_{\dataval \in \datavals} \left| \int_{\distvals} \left( \taskdist(\dataval) - \repl{\taskdist}(\dataval) \right) \hyperdistpdf(\repl{\taskdist}) ~ d\repl{\taskdist} \right| &\text{(\Cref{def:bary,def:tv})} \nonumber \\
                    &\leq \int_{\distvals} \left( \sup_{\dataval \in \datavals} \left| \taskdist(\dataval) - \repl{\taskdist}(\dataval) \right| \right) \hyperdistpdf(\repl{\taskdist}) ~ d\repl{\taskdist} & \nonumber \\
                    &= \int_{\distvals} \tv{\taskdist}{\repl{\taskdist}} \hyperdistpdf(\repl{\taskdist}) ~ d\repl{\taskdist} & \text{(\Cref{def:tv})} \nonumber \\
                    &\leq \sup_{\repl{\taskdist} \in \support{\hyperdist}}{\tv{\taskdist}{\repl{\taskdist}}} & \nonumber \\
                    &\leq \sup_{\taskdist, \repl{\taskdist} \in \support{\hyperdist} \times \support{\hyperdist}}{\tv{\taskdist}{\repl{\taskdist}}} & \nonumber \\
                    &= \mathrm{diam}\left( \hyperdist \right). &\text{(\Cref{def:diam})} \nonumber
                \end{align}
            \end{proof}

            \begin{lemma}\label{lem:eps-D}
                Given a target task distribution $\hypertargetdist$ and a source task distribution $\hypersourcedist$ with shared support $\distvals^T$ and an $\eps \in (0,1)$ satisfying \Cref{as:eps},
                $$\distshift \leq \mathrm{diam}\left( \hypersourcedist \right) +  \eps.$$
            \end{lemma}
            \begin{proof}[Proof of \Cref{lem:eps-D}]
                Given a target task distribution $\hypertargetdist$ and a source task distribution $\hypersourcedist$ with shared support $\distvals^T$ and an $\eps \in (0,1)$ satisfying \Cref{as:eps}, we have, uniformly $\forall \taskdist^t \in \distvals^T$,
                \begin{align}
                    \tv{\taskdist^t}{\bary{\hypersourcedist}} &\leq \tv{\taskdist^s_{\taskdist^t}}{\bary{\hypersourcedist}} + \tv{\taskdist^t}{\taskdist^s_{\taskdist^t}} &\text{($\taskdist^s_{\taskdist^t} \in \{ \taskdist^s \in \distvals^T ~ : ~ \tv{\taskdist^t}{\taskdist^s_{\taskdist^t}} \leq \epsilon \}$)} \label{eq:Qt-barQS} \\
                    &\leq \mathrm{diam}\left( \hypersourcedist \right) + \eps. &\text{(\Cref{lem:diam-bary,as:eps})} \label{eq:lem-D}
                \end{align}
                The non-emptiness of the set defined in line \eqref{eq:Qt-barQS} is ensured by \Cref{as:eps}.

                It then follows that,
                \begin{align}
                    \distshift &= \tv{\bary{\hypersourcedist}}{\bary{\hypertargetdist}} &\text{(\Cref{def:distshift})} \nonumber \\
                    &= \sup_{\dataval \in \datavals} \left| \int_{\distvals^T}  \left( \bary{\hypersourcedist}(\dataval) - \taskdist^t(\dataval) \right) ~ \hyperdistpdf^T(\taskdist^t) ~ d\taskdist^t \right| &\text{(\Cref{def:bary,def:tv})} \nonumber \\
                    &\leq \int_{\distvals^T} \left( \sup_{\dataval \in \datavals} \left| \bary{\hypersourcedist}(\dataval) - \taskdist^t(\dataval) \right| \right) \hyperdistpdf^T(\taskdist^t) ~ d\taskdist^t & \\
                    &= \int_{\distvals^T} \tv{\taskdist^t}{\bary{\hypersourcedist}} ~ \hyperdistpdf^T(\taskdist^t) ~ d\taskdist^t &\text{(\Cref{def:tv})} \nonumber \\
                    &\leq \int_{\distvals^T} \left( \mathrm{diam}\left( \hypersourcedist \right) + \eps \right) ~ \hyperdistpdf^T(\taskdist^t) ~ d\taskdist^t &\text{(Line \eqref{eq:lem-D})} \nonumber \\
                    &= \mathrm{diam}\left( \hypersourcedist \right) + \eps. & \nonumber
                \end{align}
            \end{proof}

            \begin{lemma}\label{lem:eps-V}
                Given a target multitask distribution $\hypertargetdist$ satisfying \Cref{as:closed-var} with p.d.f. $\hypertargetpdf ~ : ~ \distvals^T \mapsto \left[\minp^T,\maxp^T\right]$, a source task distribution $\hypersourcedist$ with p.d.f. $\hypersourcepdf ~ : ~ \distvals^T \mapsto \left[\minp^S,\maxp^S\right]$, $\minp^S > 0$, and an $\eps \in (0,1)$ satisfying \Cref{as:eps},
                $$\maxvar{\hypertargetdist} \leq \beta \left( \mathbb{V}_{\overline{\dataval}} \left[ \hypersourcedist \right] + \mathrm{vol}\left( \hypertargetdist \right) \right)$$
                for $\beta = \maxp^T/\minp^S$, $\overline{\dataval} \in \datavals$ such that $\mathbb{V}_{\overline{\dataval}} \left[ \hypertargetdist \right] = \maxvar{\hypertargetdist}$, and $\mathrm{vol}\left( \hypertargetdist \right) \coloneqq \left( \mathrm{diam}\left( \hypersourcedist \right) + \eps \right)^2$.
            \end{lemma}
            \Cref{lem:eps-V} follows from direct substitution of the result in \Cref{lem:eps-D} into the statement in \Cref{lem:V}.
            Direct substitution of the result from \Cref{lem:eps-V} into the probability the learner experiences the epistemic error margin in the statement of \Cref{thm:imperfect-distshift} completes the proof of \Cref{cor:eps}.

        \subsubsection{Alternative Total Variation Neighborhood Model}
            We here provide an epistemic error bound under an alternative total variation neighborhood model, which constrains the distance between the source and target task distributions.
        
            \begin{assumption}\label{as:eps-dist}
                There exists an $\eps \in (0,1)$ such that $\tv{\hypersourcedist}{\hypertargetdist} \leq \eps$.
            \end{assumption}

            \begin{corollary}\label{cor:eps-dist}
                Given a predictor $\preddist$, a target multitask distribution $\hypertargetdist$ satisfying \Cref{as:closed,as:closed-var} with p.d.f. $\hypertargetpdf ~ : ~ \distvals^T \mapsto \left[\minp^T,\maxp^T\right]$, a source task distribution $\hypersourcedist$ with p.d.f. $\hypersourcepdf ~ : ~ \distvals^T \mapsto \left[\minp^S,\maxp^S\right]$, $\minp^S > 0$, and an $\eps \in (0,1)$ satisfying \Cref{as:eps-dist},
                $$\trueprob{\epistemicerr \geq \margin + \bias + \convergence + \distshift} \leq \frac{\beta}{ \margin^2}\left( \mathbb{V}_{\overline{\dataval}} \left[ \hypersourcedist \right] + \eps^2 \right)$$
                for any $\margin \in \Rplus$ and where $\beta = \maxp^T / \minp^S$ and $\overline{\dataval} \in \datavals$ such that $\mathbb{V}_{\overline{\dataval}} \left[ \hypertargetdist \right] = \maxvar{\hypertargetdist}$.
            \end{corollary}

            We first prove the following lemmas, which are leveraged in the proof of \Cref{cor:eps-dist}:
            \begin{lemma}\label{lem:eps-dist-d}
                Given a source task distribution $\hypersourcedist$, a target task distribution $\hypertargetdist$, and an $\eps$ satisfying \Cref{as:eps-dist},
                $$\distshift \leq \eps.$$
            \end{lemma}
            \begin{proof}[Proof of \Cref{lem:eps-dist-d}]
                Given a source task distribution $\hypersourcedist$, a target task distribution $\hypertargetdist$, and an $\eps$ satisfying \Cref{as:eps-dist}, we have that
                \begin{align}
                    \distshift &= \tv{\bary{\hypersourcedist}}{\bary{\hypertargetdist}} &\text{(\Cref{def:distshift})} \nonumber \\
                    &= \sup_{\dataval \in \datavals}{\left| \int_{\distvals} \taskdist(\dataval) ~ \hypersourcepdf(\task) ~ d\task - \int_{\distvals} \taskdist(\dataval) ~ \hypertargetpdf(\task) ~ d\task \right|} &\text{(\Cref{def:bary,def:tv})} \nonumber \\
                    &= \sup_{\dataval \in \datavals}{\left| \int_{\distvals} \left( \taskdist(\dataval) ~ \hypersourcepdf(\task) - \taskdist(\dataval) ~ \hypertargetpdf(\task) \right) d\task \right|} \nonumber \\
                    &= \sup_{\dataval \in \datavals}{\left| \int_{\distvals} \taskdist(\dataval) \left( \hypersourcepdf(\task) - \hypertargetpdf(\task) \right) ~ d\task \right|} \nonumber \\
                    &\leq \left| \int_{\distvals} \left( \hypersourcepdf(\task) - \hypertargetpdf(\task) \right) d\task \right| &\text{($\forall \taskdist \in \distvals, \dataval \in \datavals, ~ 0 \leq \taskdist(\dataval) \leq 1$)} \nonumber \\
                    &\leq \left| \int_{\distvals^+} \left( \hypersourcepdf(\task) - \hypertargetpdf(\task) \right) d\task \right| &\text{($\distvals^+ \equiv \{ \task \in \distvals : \hypersourcepdf(\task) \geq \hypertargetpdf(\task) \}$)} \nonumber \\
                    &= \left| \int_{\distvals^+} \hypersourcepdf(\task) ~ d\task - \int_{\distvals^+} \hypertargetpdf(\task) ~ d\task \right| \nonumber \\
                    &= \left| \hypersourcedist(\distvals^+) - \hypertargetdist(\distvals^+) \right| \nonumber \\
                    &\leq \tv{\hypersourcedist}{\hypertargetdist} &\text{(\Cref{def:tv})} \nonumber \\
                    &\leq \eps. ~ &\text{(\Cref{as:eps-dist})} \nonumber
                \end{align}
            \end{proof}

            \begin{lemma}\label{lem:eps-dist-V}
                Given a target task distribution $\hypertargetdist$ satisfying \Cref{as:closed-var} with p.d.f. $\hypertargetpdf ~ : ~ \distvals^T \mapsto \left[\minp^T,\maxp^T\right]$, a source task distribution $\hypersourcedist$ with p.d.f. $\hypersourcepdf ~ : ~ \distvals^T \mapsto \left[\minp^S,\maxp^S\right]$, $\minp^S > 0$, and an $\eps \in (0,1)$ satisfying \Cref{as:eps-dist},
                $$\maxvar{\hypertargetdist} \leq \beta \left( \mathbb{V}_{\overline{\dataval}} \left[ \hypersourcedist \right] + \eps^2 \right)$$
                for $\beta = \maxp^T / \minp^S$ and $\overline{\dataval} \in \datavals$ such that $\mathbb{V}_{\overline{\dataval}} \left[ \hypertargetdist \right] = \maxvar{\hypertargetdist}$.
            \end{lemma}
            \Cref{lem:eps-dist-V} follows from direct substitution of the result in \Cref{lem:eps-dist-d} into the statement in \Cref{lem:V}.
            Direct substitution of the result from \Cref{lem:eps-dist-V} into the probability the learner experiences the epistemic error margin in the statement of \Cref{thm:imperfect-distshift} completes the proof of \Cref{cor:eps-dist}.

        \subsubsection{Bayesian Transfer Learning in Total Variation Neighborhoods}\label{ap:blt-tv-theory}
            We here provide epistemic error bounds for the setting where the constraints of both the Bayesian transfer learning setting and each of the two total variation neighborhood models are met.

            \Cref{cor:bayesian-eps} provides an epistemic error bound when \Cref{as:eps} is met.
            \begin{corollary}\label{cor:bayesian-eps}
                Given a posterior $\posterior$, a target multitask distribution $\hypertargetdist$ satisfying \Cref{as:closed,as:closed-var} with p.d.f. $\hypertargetpdf ~ : ~ \distvals^T \mapsto \left[\minp^T,\maxp^T\right]$, a source task distribution $\hypersourcedist$ with p.d.f. $\hypersourcepdf ~ : ~ \distvals^T \mapsto \left[\minp^S,\maxp^S\right]$, $\minp^S > 0$, and an $\eps \in (0,1)$ satisfying \Cref{as:eps},
                $$\trueprob{\epistemicerr \geq \margin + \bias + \CTheta + \distshift} \leq \frac{\beta}{\margin^2} \left( \mathbb{V}_{\overline{\dataval}} \left[ \hypersourcedist \right] + \mathrm{vol}\left( \hypertargetdist \right) \right)$$
                for any $\margin \in \Rplus$ and where $\CTheta \coloneqq \tv{\posterior}{\bestfitthetadist}$, $\beta = \maxp^T / \minp^S$, $\overline{\dataval} \in \datavals$ such that $\mathbb{V}_{\overline{\dataval}} \left[ \hypertargetdist \right] = \maxvar{\hypertargetdist}$, and $\mathrm{vol}\left( \hypertargetdist \right) \coloneqq \left( \mathrm{diam}\left( \hypersourcedist \right) + \epsilon \right)^2$.
            \end{corollary}
            \Cref{cor:bayesian-eps} follows directly from substitution of the result in \Cref{cor:eps} into the statement in \Cref{cor:bayesian}.

            \Cref{cor:bayesian-eps-dist} provides an epistemic error bound when \Cref{as:eps-dist} is met.
            \begin{corollary}\label{cor:bayesian-eps-dist}
                Given a posterior $\posterior$, a target multitask distribution $\hypertargetdist$ satisfying \Cref{as:closed,as:closed-var} with p.d.f. $\hypertargetpdf ~ : ~ \distvals^T \mapsto \left[\minp^T,\maxp^T\right]$, a source task distribution $\hypersourcedist$ with p.d.f. $\hypersourcepdf ~ : ~ \distvals^T \mapsto \left[\minp^S,\maxp^S\right]$, $\minp^S > 0$, and an $\eps \in (0,1)$ satisfying \Cref{as:eps-dist},
                $$\trueprob{\epistemicerr \geq \margin + \bias + \CTheta + \distshift} \leq \frac{\beta}{\margin^2} \left( \mathbb{V}_{\overline{\dataval}} \left[ \hypersourcedist \right] + \eps^2 \right)$$
                for any $\margin \in \Rplus$ and where $\CTheta \coloneqq \tv{\posterior}{\bestfitthetadist}$, $\beta = \maxp^T / \minp^S$, and $\overline{\dataval} \in \datavals$ such that $\mathbb{V}_{\overline{\dataval}} \left[ \hypertargetdist \right] = \maxvar{\hypertargetdist}$.
            \end{corollary}
            \Cref{cor:bayesian-eps-dist} follows directly from substitution of the result in \Cref{cor:eps-dist} into the statement in \Cref{cor:bayesian}.

\section{CONVERSION TO GENERALIZATION BOUNDS}\label{sec:distance}
    While \Cref{def:epistemic} defines epistemic error in terms of the total variation distance, the result in \Cref{thm:imperfect-distshift} can be immediately translated to a probabilistic bound on any function of $\preddist$ and $\targettaskdist$ that is a lower bound on some function of the total variation distance.
    \Cref{cor:cross-ent,cor:l1,cor:hellinger} provide probabilistic bounds on the cross-entropy from $\targettaskdist$ to $\preddist$ (i.e., the learner's cross-entropy loss), the $L_1$ distance between $\targettaskdist$ and $\preddist$ (i.e., the learner's $L_1$ loss), and the Hellinger distance between $\targettaskdist$ and $\preddist$, respectively.

    \paragraph{Cross-entropy}
        \Cref{cor:cross-ent} provides a probabilistic upper bound on the cross-entropy from a target task $\targettaskdist$ to the learner's predictor $\preddist$.
        It relies on the following definitions:
        \begin{definition}[Entropy, Cross-Entropy, and KL Divergence]\label{def:ent}
            Given two distinct discrete distributions $\dist$ and $\repl{\dist}$ with p.m.f.s $p$ and $\repl{p}$, respectively, and shared support $A$, $\vert A \vert \in \mathbb{N}$,
            \begin{itemize}
                \item the entropy of $\dist$ is $\ent{\dist} \coloneqq - \sum_{a \in A} \log{\left( p(a) \right)} ~ p(a)$,
                \item the cross-entropy from $\dist$ to $\repl{\dist}$ is $\ent{\dist \left| \left| \repl{\dist} \right. \right.} \coloneqq - \sum_{a \in A} \log{\left( \repl{p}(a) \right)} ~ p(a)$, and
                \item the Kullback-Leibler divergence from $\dist$ to $\repl{\dist}$ is $\kld{\dist}{\repl{\dist}} \coloneqq \sum_{a \in A} \log{\left( \frac{p(a)}{\repl{p}(a)} \right)} ~ p(a)$.
            \end{itemize}
        \end{definition}

        \noindent\textbf{\Cref{cor:cross-ent}} (Bound on the cross-entropy loss)\textbf{.}\textit{
            Given that $\lvert \datasamples \rvert \in \mathbb{N}$ (\textbf{finite sample space}), a predictor $\preddist$ with p.m.f. $\widehat{p} ~ : ~ \datasamples \mapsto \left[ \underline{v}, \overline{v} \right]$, $\underline{v} > 0$, a source task distribution $\hypersourcedist$, and a target multitask distribution $\hypertargetdist$,
            \begin{align}
                \trueprob{ \loss_{\mathrm{CE}}\left( \preddist, \targettaskdist \right) \geq \frac{2}{\underline{v}} \left( \margin + \bias + \convergence + \distshift \right)^2 + \mathcal{E} } &\leq \frac{\maxvar{\hypertargetdist}}{\margin^2} \nonumber
            \end{align}
            for any $\margin \in \Rplus$ and where $\loss_{\mathrm{CE}}\left( \preddist, \targettaskdist \right) \coloneqq \ent{\targettaskdist \left| \left| \preddist \right. \right.}$ is the cross-entropy loss and $\mathcal{E} \coloneqq \ent{\targettaskdist}$ is the entropy of the target task.
        }
        \begin{proof}[Proof of \Cref{cor:cross-ent}]
            Given that $\lvert \datasamples \rvert \in \mathbb{N}$ (\textbf{finite sample space}), a predictor $\preddist$ with p.m.f. $\widehat{p} ~ : ~ \datasamples \mapsto \left[ \underline{v}, \overline{v} \right]$, $\underline{v} > 0$, a source task distribution $\hypersourcedist$, a target multitask distribution $\hypertargetdist$, and any $\margin \in \Rplus$, we have that
            \begin{align}
                \ent{\targettaskdist \left| \left| \preddist \right. \right.} &= \kld{\targettaskdist}{\preddist} + \ent{\targettaskdist} \nonumber \\
                &\leq \frac{2}{\underline{v}} \tv{\preddist}{\targettaskdist}^2 + \ent{\targettaskdist} &\text{(Reverse Pinsker inequality \citep{sason_upper_2015})} \label{eq:reverse-pinskers} \\
                &= \frac{2}{\underline{v}} \epistemicerr^2 + \mathcal{E} &\text{(\Cref{def:epistemic})} \nonumber
            \end{align}
            and so by \Cref{thm:imperfect-distshift},
            \begin{align}
                \trueprob{ \loss_{\mathrm{CE}}\left( \preddist, \targettaskdist \right) \geq \frac{2}{\underline{v}} \left( \margin + \bias + \convergence + \distshift \right)^2 + \mathcal{E} } \leq \frac{\maxvar{\hypertargetdist}}{\margin^2} \nonumber
            \end{align}
            where $\loss_{\mathrm{CE}}\left( \preddist, \targettaskdist \right) \coloneqq \ent{\targettaskdist \left| \left| \preddist \right. \right.}$ and $\mathcal{E} \coloneqq \ent{\targettaskdist}$.
            
            In line \eqref{eq:reverse-pinskers}, the conditions needed for the reverse Pinsker inequality to hold are ensured by the requirement that $\underline{v} > 0$ and of a finite sample space.
            The condition of \Cref{thm:imperfect-distshift} that $\hypertargetdist$ satisfies \Cref{as:closed} is ensured by the requirement of a finite sample space (see \Cref{rem:closed}).
        \end{proof}
        
    \paragraph{$L_1$ distance}
        \Cref{cor:l1} provides a probabilistic upper bound on the $L_1$ distance between a target task $\targettaskdist$ and the learner's predictor $\preddist$.
        \begin{corollary}[Bound on the $L_1$ loss]\label{cor:l1}
            Given a predictor $\preddist$, a source task distribution $\hypersourcedist$, and a target multitask distribution $\hypertargetdist$ satisfying \Cref{as:closed},
            \begin{align}
                \trueprob{ \loss_{L_1}\left( \preddist, \targettaskdist \right) \geq 2 \left( \margin + \bias + \convergence + \distshift \right) } &\leq \frac{\maxvar{\hypertargetdist}}{\margin^2} \nonumber
            \end{align}
            for any $\margin \in \Rplus$ and where $\loss_{L_1}\left( \preddist, \targettaskdist \right) \coloneqq \int_{\datasamples} \left| \preddistpdf(\datasample) - \targettaskpdf(\datasample) \right| ~ d\datasample$ and $q^t$ is the p.d.f. of $\targettask$.
        \end{corollary}
        \Cref{cor:l1} is a direct consequence of the fact that $\ell_{L_1}\left( \preddist, \targettaskdist \right) = 2 \tv{\preddist}{\targettaskdist} = 2 \epistemicerr$ (\Cref{def:epistemic}).

    \paragraph{Hellinger distance}
        \Cref{cor:hellinger} provides a probabilistic upper bound on the Hellinger distance between a target task $\targettaskdist$ and the learner's predictor $\preddist$.
        \begin{corollary}[Bound on the Hellinger distance]\label{cor:hellinger}
            Given a predictor $\preddist$, a source task distribution $\hypersourcedist$, and a target multitask distribution $\hypertargetdist$ satisfying \Cref{as:closed},
            \begin{align}
                \trueprob{ \mathrm{d}_{\mathrm{H}^2} \left( \preddist, \targettaskdist \right) \geq \margin + \bias + \convergence + \distshift} &\leq \frac{\maxvar{\hypertargetdist}}{\margin^2} \nonumber
            \end{align}
            for any $\margin \in \Rplus$ and where $\mathrm{d}_{\mathrm{H}^2} \left( \preddist, \targettaskdist \right) \coloneqq \frac{1}{2} \int_{\datasamples} \left( \sqrt{\preddistpdf\left( \datasample \right)} - \sqrt{\targettaskpdf\left( \datasample \right)} \right)^2 ~ d\datasample$ and $q^t$ is the p.d.f. of $\targettask$.
        \end{corollary}
        \Cref{cor:hellinger} is a direct consequence of the fact that $\mathrm{d}_{\mathrm{H}^2} \left( \preddist, \targettaskdist \right) \leq \tv{\preddist}{\targettaskdist} = \epistemicerr$ (\Cref{def:epistemic}).
        \begin{remark}
            The Hellinger distance between the target DGP and the learner's predictive distribution is an important measure in Bayesian asymptotic theory.
            In the setting in which the predictive distribution is formed on the basis of data arising from the target DGP (i.e., there is no distribution shift), the conditions for \textnormal{Hellinger consistency} are those under which (almost surely) the learner eventually assigns probability 0 to any set of DGPs that excludes a neighborhood of the target DGP \citep{walker_modern_2004}.
            In our setting (imperfect multitask learning setting with distribution shift), \Cref{cor:hellinger} provides the probability that the learner's predictor is the given margin away from the target DGP.
            In a Bayesian learning setting, \Cref{cor:hellinger} is therefore related to the conditions under which Hellinger consistency is violated.
        \end{remark}

\section{EMPIRICAL DEMONSTRATIONS}\label{ap:experiments}
    \newcommand{\covariate}{\boldsymbol{\xi}}
    \newcommand{\tvub}[2]{\overline{\mathrm{d_{TV}}}\left( #1, #2 \right)}

    All computations were run on a laptop using a single CPU.

    \subsection{Bayesian Transfer Learning in Total Variation Neighborhoods}\label{ap:btl-tv}

    We here empirically (i) illustrate the conclusions from the results in \Cref{sec:cases} and (ii) demonstrate the relationship between the looseness in the epistemic error margin and the presence of negative transfer (discussed in \Cref{sec:negtransfer}).

    Since the total variation distance is in general computationally intractable, we approximate it with the upper bound (established by Pinsker's inequality):
    $$\tv{P}{\repl{P}} \leq \tvub{P}{\repl{P}} \coloneqq \sqrt{\frac{\kld{P}{\repl{P}}}{2}}$$
    where $\kld{P}{\repl{P}}$ is the Kullback-Leibler divergence from a distribution $P$ to $\repl{P}$ as defined in \Cref{def:ent}.\footnote{
        To compute $\kld{P}{\repl{P}}$ from a continuous distribution $\dist$ to a continuous distribution $\repl{P}$, we used a numerical approximation on the basis of 400 samples from $P$.
    }

    \paragraph{Setting}
        In the Bayesian linear regression setting, the source data was generated as
        $$\sourcedata \sim \mathrm{N}\left( \beta^S_1\covariate_{(1:n,1)} + \beta^S_2\covariate_{(1:n,2)}, \boldsymbol{\sigma}^S_{(1:n)} \right)$$
        given a set of known covariate values $\covariate_{(1:n)}$.
        Each source task was characterized by the same value of $\beta^S = \left( \beta^S_1, \beta^S_2 \right)$ and a different value of $\sigma^S$ ($\boldsymbol{\sigma}^S_{(1:n)} = [ \sigma^S_1, \ldots{}, \sigma^S_n ]$).
        Values of $\sigma^S$ (distinct source tasks) were sampled independently from an inverse gamma distribution with concentration parameter $\alpha_S = 20$ and rate parameter $\delta_S = 10$.
        In all cases, the learner observed only one data point, corresponding to covariates each drawn from a standard uniform distribution, from each source task.

        Data in each target task followed a data-generating distribution of the same form:
        $$\datasample \sim \mathrm{N}\left( \beta^T_1\xi_1 + \beta^T_2\xi_2, \sigma^T \right)$$
        where $\beta^T = \left( \beta^T_1, \beta^T_2 \right)$ not necessarily equal to $\beta^S$.
        In all cases, the learner observed one data point corresponding to covariates $\xi_1 = \xi_2 = 1$ when encountering the target task.

        The learner estimates a Normal-Inverse-Gamma model
        $$\beta \vert \sigma^2 \sim \mathrm{N}\left( \beta^0, \sigma^2B_0 \right), \sigma^2 \sim \mathrm{IG}\left( \alpha_0, \delta_0 \right)$$
        with $\beta^0 = [0, 0]$, $B_0 = \begin{bmatrix} 1 & 0 \\ 0 & 1 \end{bmatrix}$, $\alpha_0 = 20$, and $\delta_0 = 10$.
        The learner is aware that the source tasks are sampled from a distribution, and so updates only its distribution over $\beta$ on the basis of the source data.
        Notice the learner's prior (and posterior) distribution over $\sigma^2$ is the same as the distribution among the source tasks, i.e., there is no approximation bias ($\bias = 0$).

    \paragraph{Posterior convergence and neighborhood size}
        \Cref{fig:tv-neighborhoods} illustrates the conclusions from the results in \Cref{sec:cases}.
        In \Cref{fig:convergence,fig:epsilon}, each target task is constrained to fall within an $\epsilon$-neighborhood of a source task.
        Neighborhood size $\epsilon$ is varied across simulations.
        On each simulation, 10 values of $\sigma^S$ (distinct source tasks) are randomly sampled from the source task distribution described above to generate the source data ($\beta^S = [0,1]$).
        To generate synthetic target tasks, we select at random (i) an index $i \in [1, \ldots{}, 10]$ to indicate a source task from among the source tasks encountered by the learner, and (ii) a value $\repl{\epsilon}$ uniformly between 0 and the specified neighborhood size.
        A target task distance $\repl{\epsilon}$ away from the $i^{\mathrm{th}}$ source task was constructed by setting $\sigma^T = \sigma^S_i$ and adjusting $\beta^T$.
        \newcommand{\figwidth}{.85\linewidth}
        \begin{figure}[t!]
            \begin{subfigure}{.32\linewidth}
                \includegraphics[width=\linewidth]{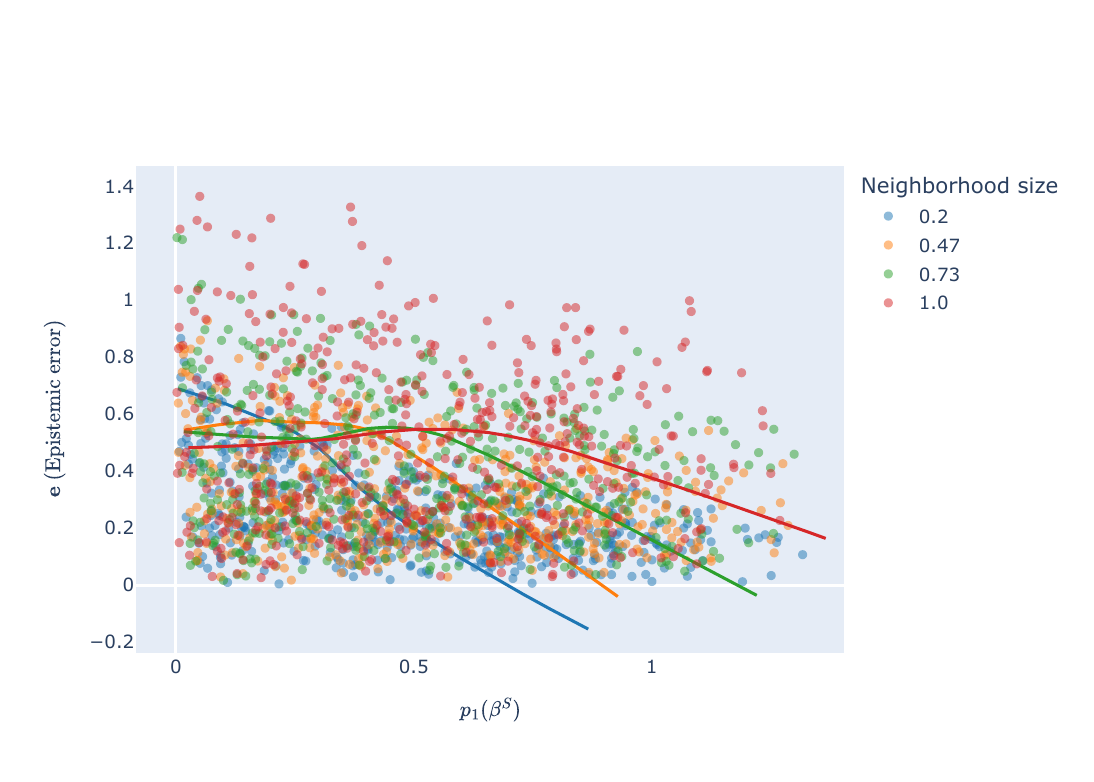}
                \caption{Effect of convergence of the posterior distribution.
                    LOWESS trend lines for data at each neighborhood size are visualized in the corresponding color.    
                }
                \label{fig:convergence}
            \end{subfigure}\hfill\begin{subfigure}{.32\linewidth}
                \includegraphics[width=\linewidth]{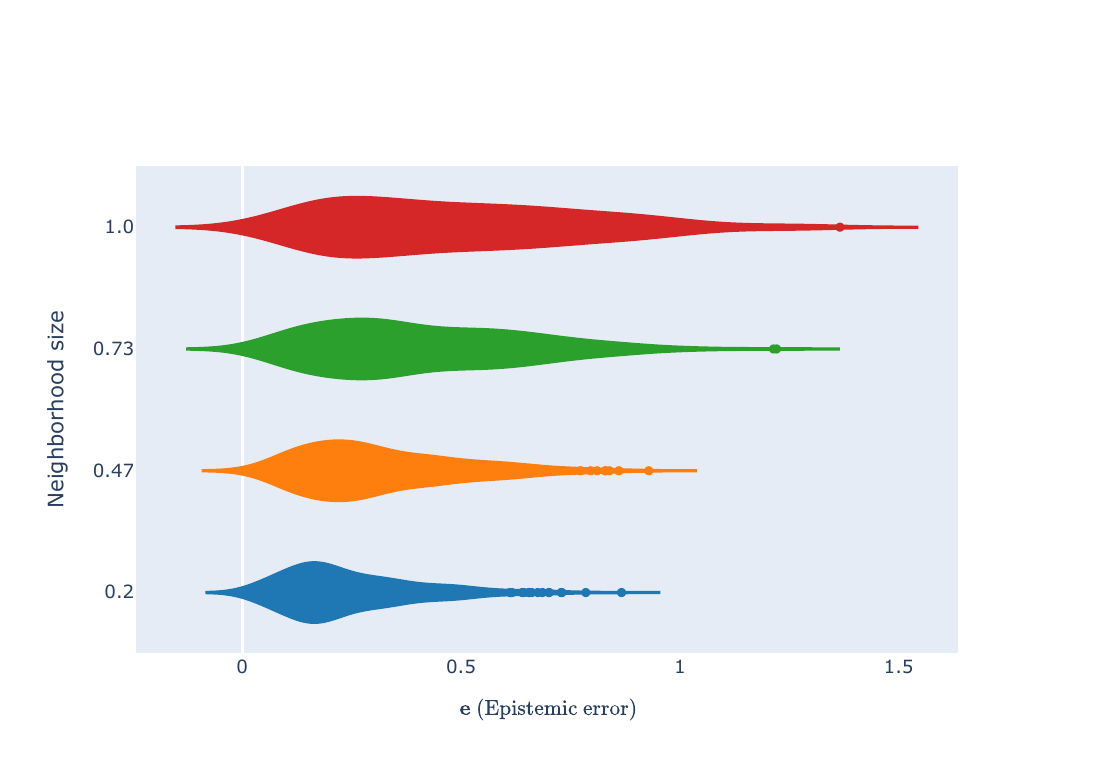}
                \caption{Effect of neighborhood size. \\ ~ \\ ~ \\ ~ \\ ~}
                \label{fig:epsilon}
            \end{subfigure}\hfill\begin{subfigure}{.32\linewidth}
                \includegraphics[width=\linewidth]{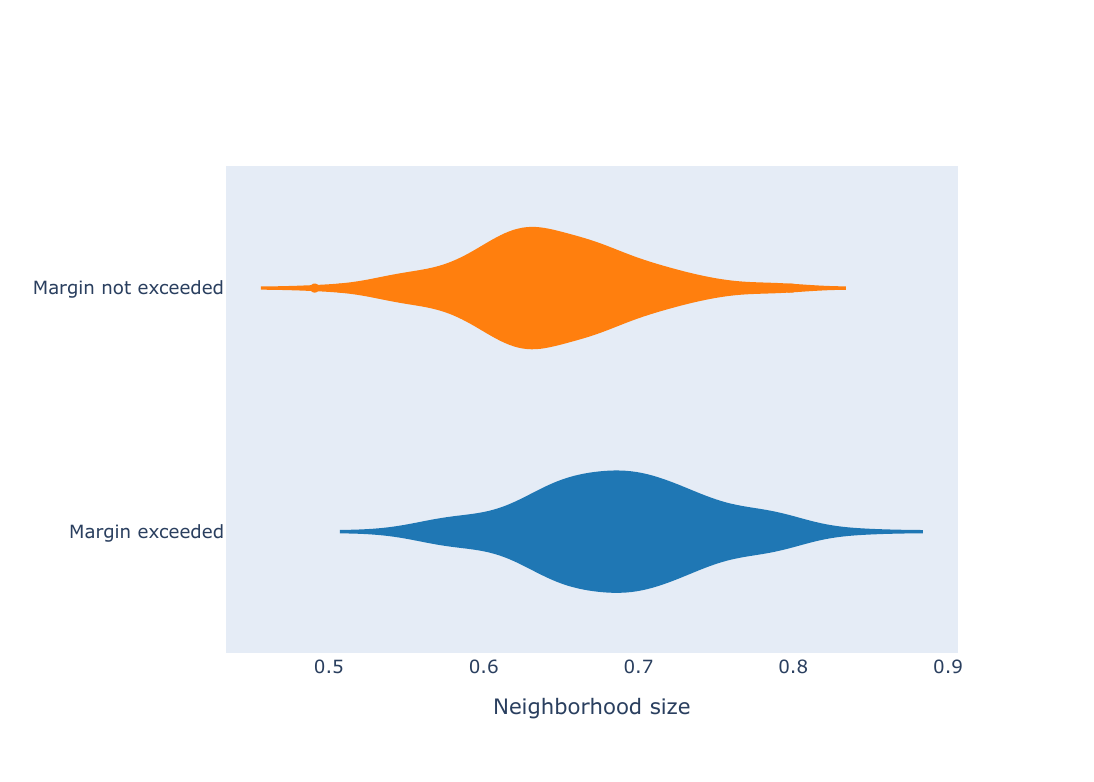}
                \caption{Effect of neighborhood size on the probability of exceeding the epistemic error margin. \\ ~ \\ ~}
                \label{fig:epsilon-delta}
            \end{subfigure}
            \caption{(a--b) Epistemic error $\epistemicerr$ as a function of convergence of the posterior distribution and neighborhood size.
                For each neighborhood size $\epsilon$, the plots show 500 simulations where variation is with respect to the set of source tasks and target task sampled from their respective task distributions.
                (c) The probability of exceeding the epistemic error margin as a function of neighborhood size.
                The plots show 500 simulations where variation is with respect to the target task sampled from the target task distribution.
            }
            \label{fig:tv-neighborhoods}
        \end{figure}

        \Cref{fig:convergence} shows that epistemic error decreases as a function of the probability the learner's posterior assigns to the source data-generating value $\beta^S$.
        This illustrates the conclusion from \Cref{cor:bayesian} that the epistemic error margin increases as a function of $\CTheta \coloneqq \tv{\posterior}{\bestfitthetadist}$ (since the best approximate parameter distribution assigns probability 1 to a small neighborhood of $\beta^S$, this tracks closely with the complement of $\CTheta$).

        \Cref{fig:epsilon} shows that epistemic error increases as a function of the neighborhood size.
        While these results are reminiscent of the conclusion from \Cref{cor:eps} that the probability the learner exceeds a given epistemic error margin increases as a function of neighborhood size $\epsilon$, they could arise from the effect of neighborhood size on the epistemic error margin, i.e., the effect of $\epsilon$ on the extent of distribution shift $\distshift$, as opposed to on target task variability $\maxvar{\hypertargetdist}$ (recall from \Cref{lem:eps-D} that the extent of distribution shift increases as a function of $\epsilon$).
        
        \Cref{fig:epsilon-delta} directly tests the effect of neighborhood size on the probability of exceeding a given epistemic error margin.
        Here, target tasks are sampled from a given target multitask distribution.
        Each target task was characterized by the same value of $\beta^T$ ($\beta^T = [1,1]$; this corresponds to the setting of positive transfer in \Cref{fig:pos-exp}) and a different value of $\sigma^T$ sampled independently from an inverse gamma distribution with concentration parameter $\alpha_T = 20$ and rate parameter $\delta_T = 10$.
        To isolate the effect of task variability, we also imposed perfect learning ($\preddist = \bary{\hypersourcedist}$).
        The $x$-axis of \Cref{fig:epsilon-delta} shows values of $\tvub{\targettask}{\sourcetask}$ across simulations, where $\sourcetask \sim \hypersourcedist$ is randomly sampled on each simulation; this ensures that \Cref{as:eps} is satisfied for $\epsilon$ equal to the value indicated on the $x$-axis.
        The $y$-axis shows whether or not the epistemic error margin was exceeded on each simulation.
        \Cref{fig:epsilon-delta} shows that the probability the learner exceeds a given epistemic error margin increases as a function of neighborhood size.
        
        \begin{figure}[h!]
            \begin{subfigure}{.32\linewidth}
                \centering
                \begin{tikzpicture}
                    \node (P0) {$\widehat{P}_0$};
                    \node[circle,minimum size=.5in,draw=black,dashed,below=of P0] (S) {$\bary{\hypersourcedist}$};
                    \node[] at ($(P0.center)!.33!(S.center)$) (P1) {$\widehat{P}_1$};
                    \node[] at ($(P0.center)!.66!(S.center)$) (P2) {$\widehat{P}_2$};
                    \node[right=of S] (T){$\targettask$};

                    \draw[black] (P0) -- (T.center);
                    \draw[black] (P1) -- (T.center);
                    \draw[black] (P2) -- (T.center);
                \end{tikzpicture}
                \caption{Positive transfer. \\ ~}
                \label{fig:pos}
            \end{subfigure}\hfill\begin{subfigure}{.32\linewidth}
                \centering
                \begin{tikzpicture}
                    \node (P0) {$\widehat{P}_0$};
                    \node[circle,minimum size=.5in,draw=black,dashed,below left=of P0] (S) {$\bary{\hypersourcedist}$};
                    \node[] at ($(P0.center)!.33!(S.center)$) (P1) {$\widehat{P}_1$};
                    \node[] at ($(P0.center)!.66!(S.center)$) (P2) {$\widehat{P}_2$};
                    \node[below right=of P0,xshift=-.2cm,yshift=-.1cm] (T){$\targettask$};

                    \draw[black] (P0) -- (T.center);
                    \draw[black] (P1) -- (T.center);
                    \draw[black] (P2) -- (T.center);
                \end{tikzpicture}
                \caption{Negative transfer. \\ ~}
                \label{fig:neg}
            \end{subfigure}\hfill\begin{subfigure}{.32\linewidth}
                \centering
                \begin{tikzpicture}
                    \node (P0) {$\widehat{P}_0$};
                    \node[circle,minimum size=.5in,draw=black,dashed,below=of P0] (S) {$\bary{\hypersourcedist}$};
                    \node[] at ($(P0.center)!.33!(S.center)$) (P1) {$\widehat{P}_1$};
                    \node[] at ($(P0.center)!.66!(S.center)$) (P2) {$\widehat{P}_2$};
                    \node[right=of P1,xshift=1mm] (T){$\targettask$};

                    \draw[black] (P0) -- (T.center);
                    \draw[black] (P1) -- (T.center);
                    \draw[black] (P2) -- (T.center);
                \end{tikzpicture}
                \caption{Positive followed by negative transfer.}
                \label{fig:posneg}
            \end{subfigure}
            \caption{\Cref{fig:pos,fig:neg,fig:posneg} show schematically how additional learning from the source data (decreased distance of a predictor from $\bary{\hypersourcedist}$) can have different effects on the epistemic error margin (distance to a $\targettaskdist$, indicated by the solid lines).
            }
            \label{fig:negtransfer}
        \end{figure}
        
        \begin{figure}[h!]
            \begin{subfigure}{.32\linewidth}
        		\includegraphics[width=\figwidth]{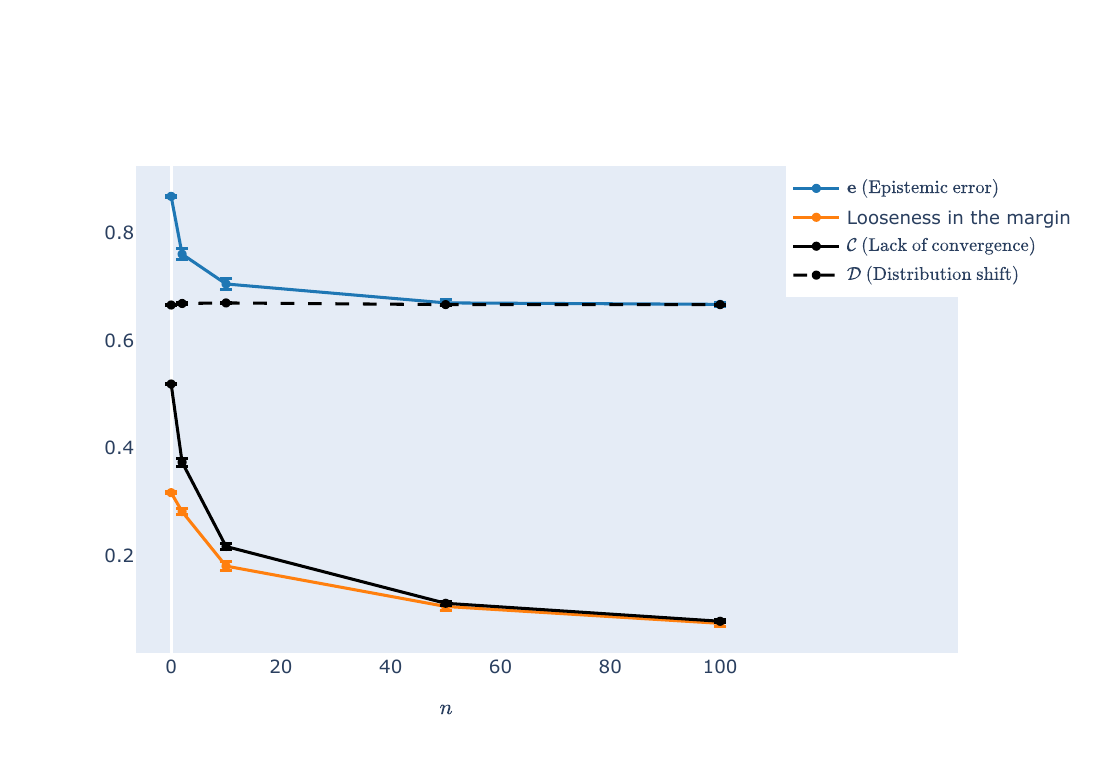}
        		\caption{Positive transfer (represented in \Cref{fig:pos}): $\beta^S = [0,1], \beta^T = [1,1]$. \\ ~}
        		\label{fig:pos-exp}
        	\end{subfigure}\hfill\begin{subfigure}{.32\linewidth}
        		\includegraphics[width=\figwidth]{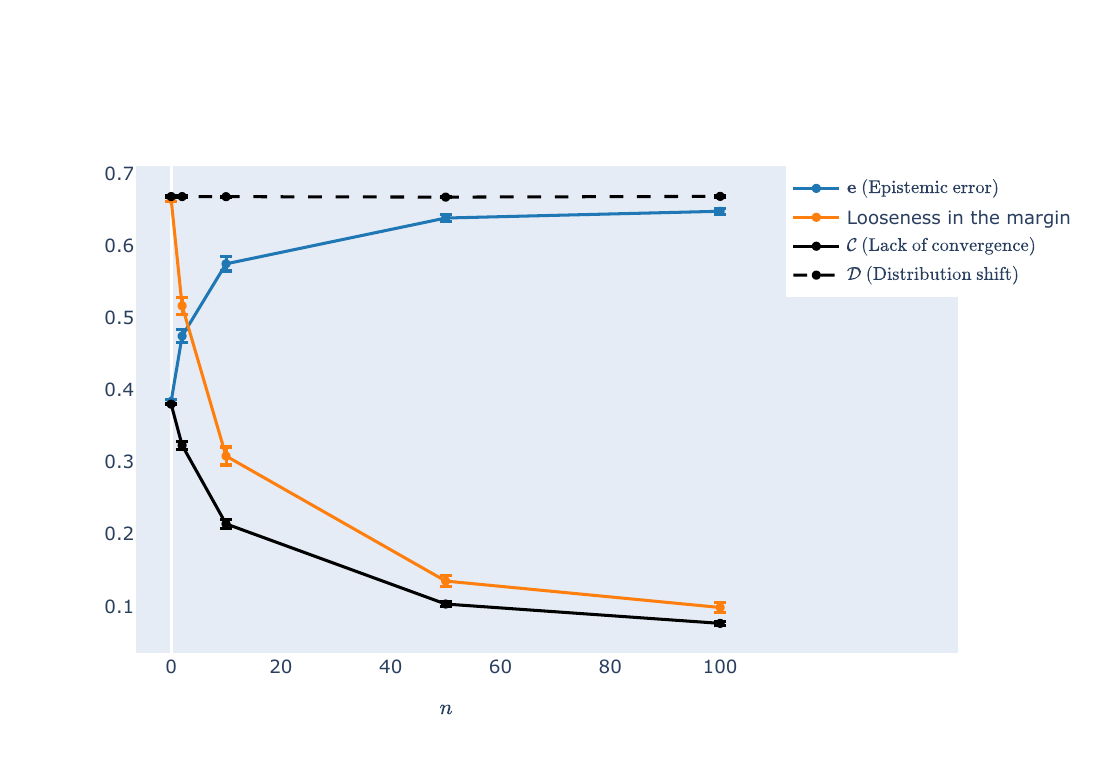}
        		\caption{Negative transfer (represented in \Cref{fig:neg}): $\beta^S = [-1,0], \beta^T = [1,0]$. \\ ~}
        		\label{fig:neg-exp}
        	\end{subfigure}\hfill\begin{subfigure}{.32\linewidth}
        		\includegraphics[width=\figwidth]{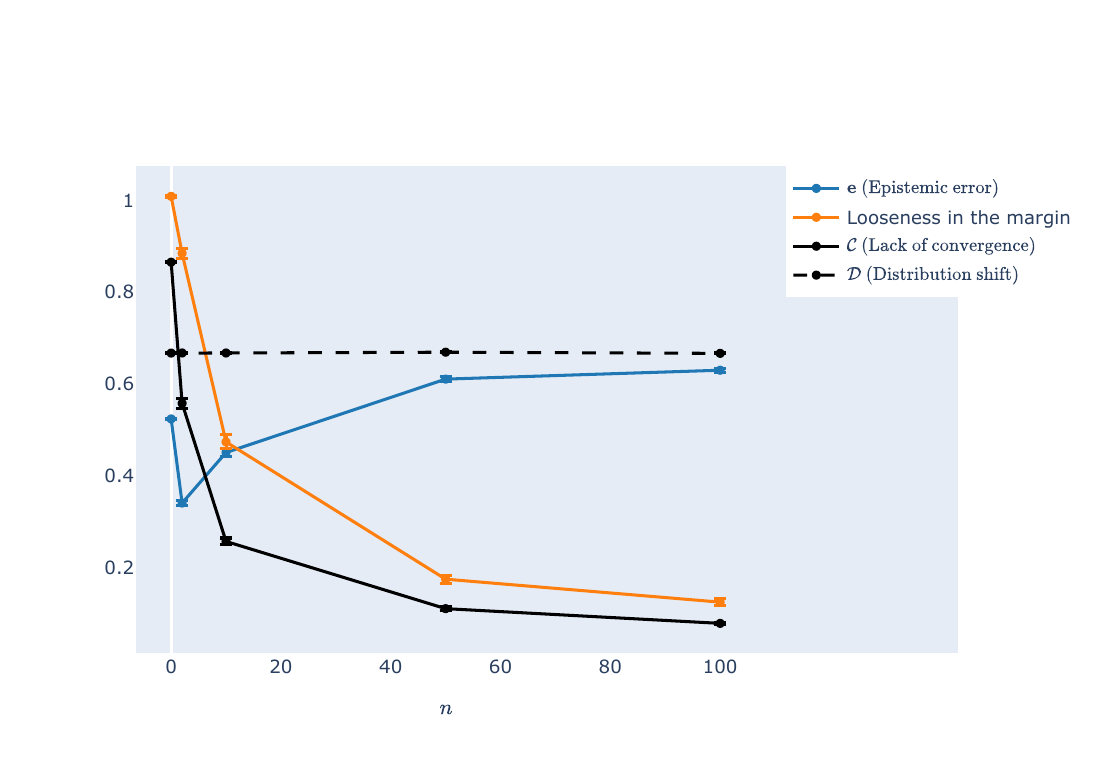}
        		\caption{Positive followed by negative transfer (represented in \Cref{fig:posneg}): $\beta^S = [0,2], \beta^T = [0,1]$.}
        		\label{fig:posneg-exp}
        	\end{subfigure}
        	\caption{Epistemic error $\epistemicerr$ incurred in synthetic settings designed to reflect the settings in \Cref{fig:pos,fig:neg,fig:posneg}, respectively (the horizontal axes of \Cref{fig:pos,fig:neg,fig:posneg} represent the value of $\beta_1$ and the vertical axes represent the value of $\beta_2$).
        		Looseness in the epistemic error margin is computed as $\left( \tvub{\preddist}{\bestfitdist{S}} + \tvub{\bary{\hypersourcedist}}{\bary{\hypertargetdist}} \right) - \tvub{\preddist}{\targettask} \approx \left( \convergence + \distshift \right) - \epistemicerr$ (recalling that in this case $\bestfitdist{S} = \bary{\hypersourcedist}$).
        		Lines and error bars denote means and standard errors, respectively, across 500 simulations, where variation is with respect to the set of source tasks and target task sampled from their respective task distributions.
        	}
        	\label{fig:experiments-negtransfer}
        \end{figure}

    \paragraph{Negative transfer}
        \Cref{sec:negtransfer} defined the occurrence of \textit{negative transfer} as a case where predictors closer to the source data distribution $\bary{\hypersourcedist}$ commit larger epistemic errors than those further away.
        \Cref{fig:negtransfer} shows some examples of when such negative transfer might occur.
        In \Cref{fig:neg,fig:posneg}, the
        learner is expected to (eventually) experience negative transfer; in these cases, $\preddist_0$ is closer to $\targettask$ than $\bary{\hypersourcedist}$, and so learning from the source data increases the epistemic error.
        
        \Cref{fig:experiments-negtransfer} demonstrates the relationship between, on one hand, the looseness in the epistemic error margin and, on the other, the presence of negative transfer.
        It shows the results of experiments using the Bayesian linear regression setting described above, and where the source and target task distributions are set to correspond to the schematics shown in \Cref{fig:negtransfer}.
        On each simulation, $n$ values of $\sigma^S$ (distinct source tasks) are randomly sampled from the source task distribution described above to generate the source data, where the value of $n$ is indicated on the $x$-axes and the value of $\beta^S$ in the subfigure captions.
        To generate synthetic target tasks, values of $\sigma^T$ (distinct target tasks) were sampled independently from an inverse gamma distribution with concentration parameter $\alpha_T = 20$ and rate parameter $\delta_T = 10$.

        \Cref{fig:experiments-negtransfer} shows that the looseness in the epistemic error margin (orange line) is very high (at or above the extent of distribution shift $\distshift$; dashed black line) only in the cases where the learner will experience negative transfer, i.e., increases in epistemic error (blue line) as a function of $n$.
        As the learner acquires more source data, its lack of convergence ($\convergence$; solid black line) generally declines: The learner's predictor converges on the true source data distribution.
        As $\convergence$ declines, the only remaining sources of epistemic error are distribution shift and target task variability (recalling that approximation bias $\bias = 0$).
        The extent of distribution shift far outweighs the extent of task variability: As it acquires more source data, the learner's epistemic error converges to the degree of distribution shift.
        As an additional consequence of the learner's predictor converging on the true source data distribution (decreases in $\convergence$), the looseness in the epistemic error margin (which is due to the discrepancy between $\tv{\preddist}{\bary{\hypertargetdist}}$ and $\convergence + \distshift$) also declines.

    \subsection{Iris Data}\label{ap:iris}
        \newcommand{\sourceflip}{\zeta_{\mathrm{source}}}
        \newcommand{\targetflip}{\zeta_{\mathrm{target}}}
    
        We now present results in a setting constructed using the Iris data set from the UCI Machine Learning Repository \citep{iris_53} to (i) illustrate the implications of \Cref{thm:imperfect-distshift} and (ii) again demonstrate the relationship between the looseness in the epistemic error margin and negative transfer.

        The Iris data set contains data from 147 flowers (observations).
        For each flower, four measurements (sepal length, sepal width, petal length, and petal width) and species are recorded, where species is one of ``setosa'', ``versicolor'', or ``virginica''.
        The learner's goal is to predict a flower's species on the basis of the four measurements.
    
        \paragraph{Setting}
            We constructed source data sets from data accruing to each of ten source ``tasks'', where each task generated a corruption of the original Iris data set.
            To generate data from a given source task, we flipped the labels of the original data set with a pre-specified probability $\sourceflip$.
            From the set of data from all ten source tasks, the learner estimates a multinomial logistic regression model, i.e., models the data as
            $$\dataval_i \sim \mathrm{Multinomial}\left( \{ p(\dataval_i = \datasample \mid \boldsymbol{\xi}_i) ~ : ~ \datasample \in \{ \mathrm{setosa}, \mathrm{versicolor}, \mathrm{virginica} \} \} \right)$$
            where $\boldsymbol{\xi} \in \mathbb{R}^{1470 \times 4}$ contains the four measurement values for each of the 10 tasks generating 147 observations.
            We generated target data by flipping labels of the original Iris data set with a pre-specified probability $\targetflip$.
            
            In this setting, we can exactly compute $\epistemicerr$, $\distshift$, $\maxvar{\hypertargetdist}$, and $\tv{\preddist}{\bary{\hypersourcedist}}$ (which, by the triangle inequality, is a lower bound on $\bias + \convergence$).
            We manipulate the degree of distribution shift $\distshift$ by manipulating $\sourceflip$ and $\targetflip$ (when $\sourceflip \neq \targetflip$, distribution shift has occurred).

            This is a supervised learning setting where the learner aims to predict outputs $\datasample$ given inputs $\xi$.
            In other words, the learner's predictive distribution $\preddist$ depends on the value of the measurements $\xi$.
            Thus, $\epistemicerr$ and $\tv{\preddist}{\bary{\hypersourcedist}}$ are computed conditional on a value $\xi$.
            This reflects a situation where at test time, the learner is tasked with making a prediction given a single set of four measurements.

        \paragraph{Illustrating \Cref{thm:imperfect-distshift}}
            \Cref{fig:iris-thm1} shows that the empirical probability of exceeding the epistemic error margin almost never exceeds the theoretical value provided in \Cref{thm:imperfect-distshift}.\footnote{Exceptions occur when $\delta^{\star}$ approaches 1, where $\hat{\delta}$ appears to fall even above the shaded approximate 95\% confidence region shown in \Cref{fig:iris-thm1}.
                However, this is not necessarily indicative of ``surprisingly'' high values of $\hat{\delta}$: The normal approximation used to construct the approximate confidence region is especially inaccurate for $\delta^{\star} \approx 1$.
            }
            In this figure, we plot the theoretical value of the probability of exceeding the epistemic error margin ($\delta^{\star}$; $x$-axis) against the proportion of tasks for which the epistemic error margin was actually exceeded ($\hat{\delta}$; $y$-axis).
            On each of 200 runs, we selected values $\sourceflip$ and $\targetflip$ separately and uniformly at random between 0 and 1.
            The source data was constructed from the original Iris data set as described above.
            \begin{figure}[h!]
                \begin{subfigure}{.48\linewidth}
        		      \includegraphics[width=\figwidth]{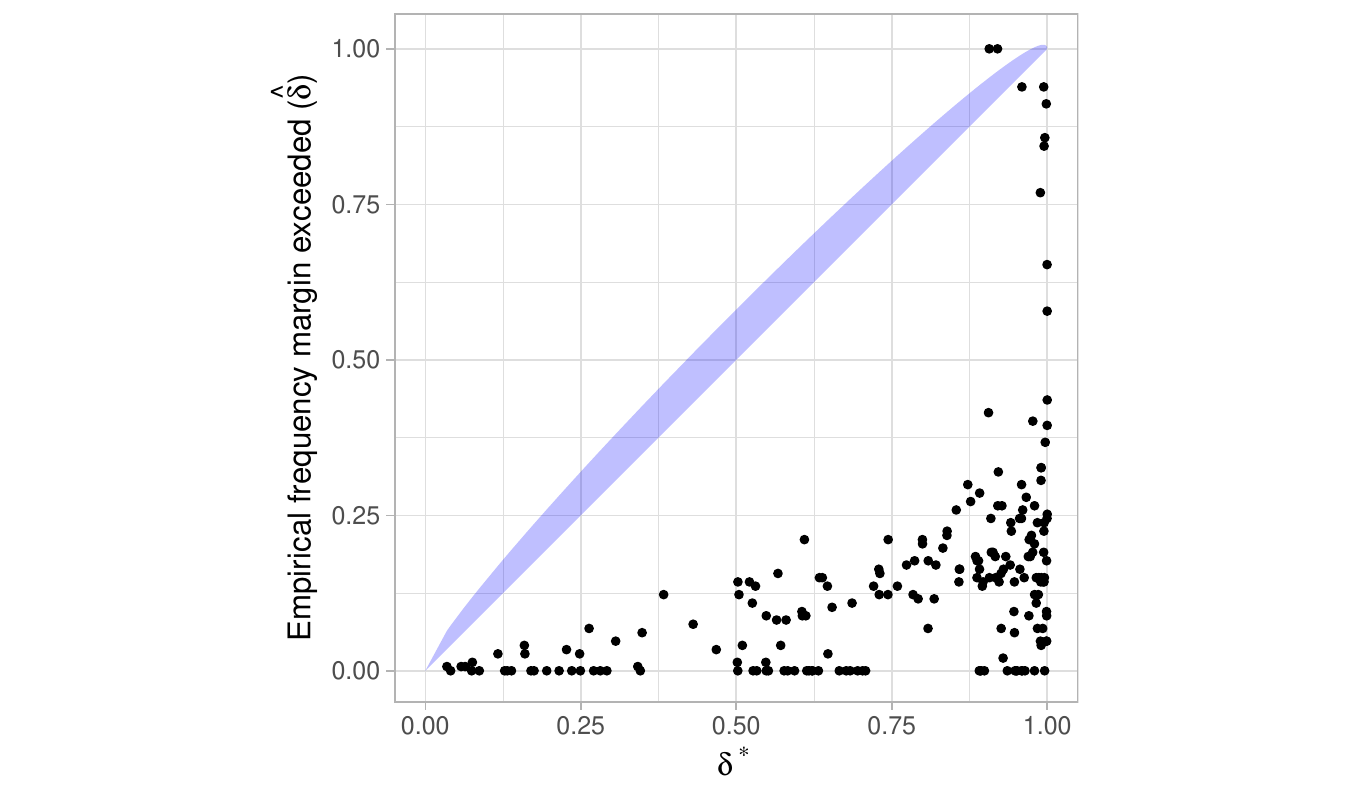}
        		      \caption{Theoretical ($\delta^{\star}$) vs. empirical ($\hat{\delta}$) probability of exceeding the epistemic error margin.
                      $\delta^{\star}$ is computed as $\frac{\maxvar{\hypertargetdist}}{\alpha^2}$.
                        $\hat{\delta}$ is computed as the proportion of the 147 target tasks for which the epistemic error margin was exceeded.
                        The shaded region denotes approximately two standard errors above the value of $\delta^{\star}$ (computed according to a normal approximation; note that this approximation is least accurate for values $\delta^{\star}$ very high or very low, where the bound appears to be exceeded).
                    }
        		      \label{fig:iris-thm1}
        	    \end{subfigure}\hfill\begin{subfigure}{.48\linewidth}
        		      \includegraphics[width=\figwidth]{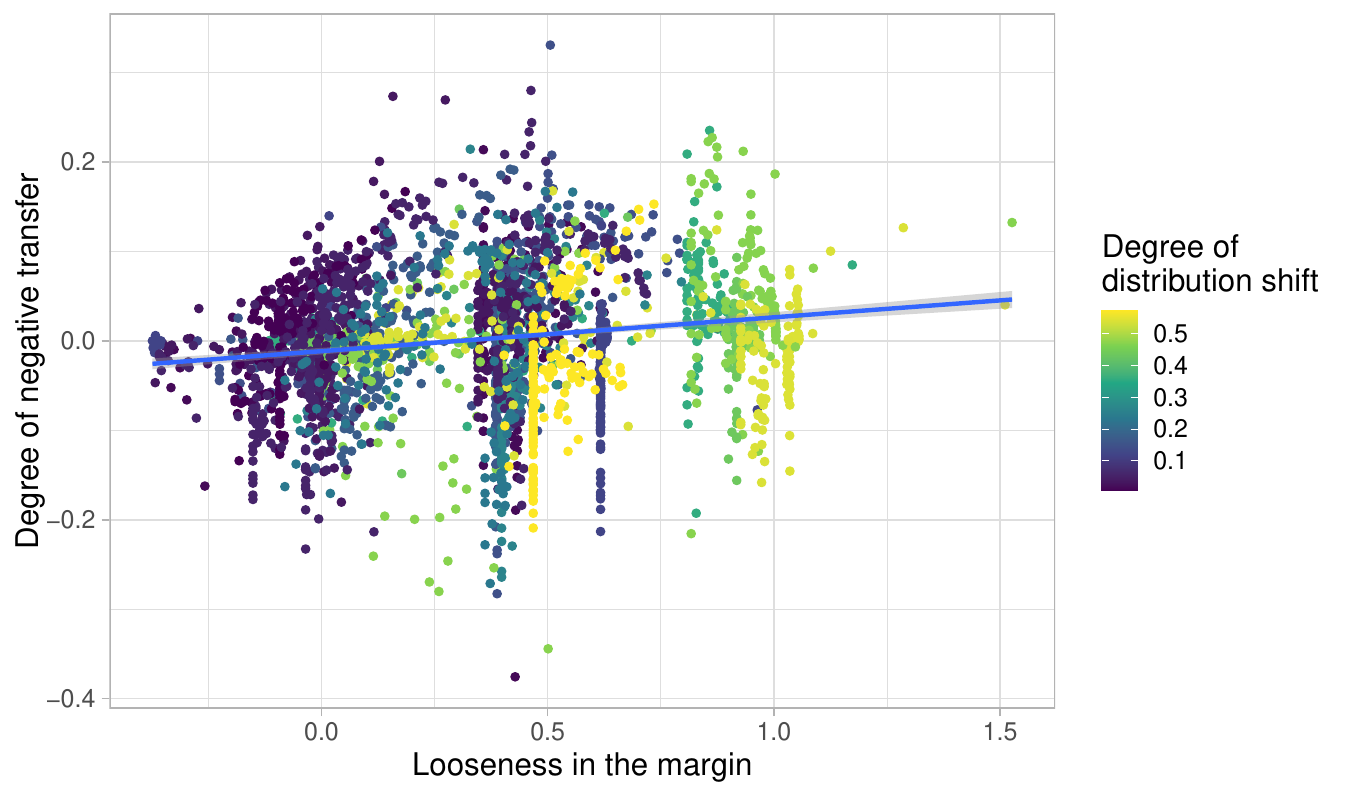}
        		      \caption{Looseness in the epistemic error margin is computed as $\tv{\preddist_{\mathrm{small} ~ N}}{\bary{\hypersourcedist}} + \distshift + \alpha - \epistemicerr_{\mathrm{small} ~ N}$.
                        Degree of negative transfer is computed as $\epistemicerr_{\mathrm{large} ~ N} - \epistemicerr_{\mathrm{small} ~ N}$.
                        Color indicates the value of $\distshift$.
                        The line and shaded region show the best-fitting linear trend and 95\% confidence region of that trend. \\ ~ \\ ~ \\
                    }
        		      \label{fig:iris-negtransfer}
        	    \end{subfigure}
        	    \caption{Results in a setting constructed using the Iris data set \citep{iris_53} to (a) illustrate \Cref{thm:imperfect-distshift} and (b) demonstrate the relationship between looseness in the epistemic error margin and negative transfer.
                    In both panels, $\alpha = \frac{1}{2}$.
                }
        	    \label{fig:iris}
            \end{figure}
            
            On each of these runs, the sample space of the target task distribution is a set of three tasks: $Q^{\mathrm{uncorrupted}}$, $Q^{\mathrm{corruption1}}$ and $Q^{\mathrm{corruption2}}$ assign probability 1 to the original label and each of the other two labels, respectively.
            The learner thus encountered $Q^{\mathrm{uncorrupted}}$ with probability $1 - \frac{2}{3} \targetflip$ and each of $Q^{\mathrm{corruption1}}$ and $Q^{\mathrm{corruption2}}$ with probability $\frac{1}{3} \targetflip$.

            On each of these runs, we considered 147 predictors: one for each set of measurements in the original Iris data set.
            Values $\hat{\delta}$ ($y$-axis) are the proportion of these 147 target tasks for which the theoretical epistemic error margin was exceeded.

        \paragraph{Negative transfer}
            While \Cref{fig:iris-thm1} shows that the epistemic error bound provided by \Cref{thm:imperfect-distshift} is valid, this figure also shows that the bound is in some cases considerably loose: The empirical probability of exceeding the epistemic error margin $\hat{\delta}$ is often much smaller than the theoretical probability $\delta^{\star}$.
            Recall from \Cref{sec:negtransfer,ap:btl-tv} that the degree of looseness in the epistemic error margin is conceptually related to negative transfer, a situation where learning from source data increases epistemic error.
            
            \Cref{fig:iris-negtransfer} shows the empirical relationship between the degree of looseness in the epistemic error margin ($x$-axis) and of negative transfer ($y$-axis).
            To assess the degree of negative transfer, we compare two learners: A learner trained on data from one task ($N = 147$) forms predictor $\preddist_{\mathrm{small} ~ N}$ and commits epistemic error $\epistemicerr_{\mathrm{small} ~ N}$.
            A learner trained on data from 10 tasks ($N = 1,470$) forms predictor $\preddist_{\mathrm{large} ~ N}$ and commits epistemic error $\epistemicerr_{\mathrm{large} ~ N}$.
            We computed the degree of negative transfer as $\epistemicerr_{\mathrm{large} ~ N} - \epistemicerr_{\mathrm{small} ~ N}$, i.e., as the extra epistemic error incurred by the learner who has learned more from the source data.

            \Cref{fig:iris-negtransfer} shows that the degree of looseness in the margin of $\epistemicerr_{\mathrm{small} ~ N}$ is positively correlated with the degree of negative transfer.
            The size of the epistemic error margin is affected by factors such as distribution shift, which contributes to the degree of negative transfer \citep{ben-david_analysis_2006,ben-david_theory_2010,zhang_survey_2023}.
            To test whether the relationship between looseness in the margin and the degree of negative transfer held when controlling for the degree of distribution shift, we fit a linear model of the degree of negative transfer as a function of the looseness in the epistemic error margin (with estimated effect $\widehat{\beta
            }_{l}$) and degree of distribution shift (with effect $\widehat{\beta
            }_{\mathcal{D}}$): The estimated effect of looseness remains positive when distribution shift is controlled for ($\widehat{\beta
            }_{l} = .06$ ($SE = .004$) and $\widehat{\beta
            }_{\mathcal{D}} = -.07$ ($SE = .008$)).
            In other words, learning from data may benefit the learner even in the presence of distribution shift (see schematic example in \Cref{fig:pos}), and vice versa.

\bibliographystyle{apalike}
\bibliography{bibliography}

%% file: preprint.bbl
\begin{thebibliography}{}

\bibitem[Baxter, 2000]{baxter_model_2000}
Baxter, J. (2000).
\newblock A model of inductive bias learning.
\newblock {\em Journal of Artificial Intelligence Research}, 12.

\bibitem[Ben-David et~al., 2010]{ben-david_theory_2010}
Ben-David, S., Blitzer, J., Crammer, K., Kulesza, A., Pereira, F., and {Wortman
  Vaughan}, J. (2010).
\newblock A theory of learning from different domains.
\newblock {\em Machine Learning}, 79.

\bibitem[Ben-David et~al., 2006]{ben-david_analysis_2006}
Ben-David, S., Blitzer, J., Crammer, K., and Pereira, F. (2006).
\newblock Analysis of representations for domain adaptation.
\newblock In {\em Proceedings of the 20th Conference on Neural Information
  Processing Systems (NIPS 2006)}.

\bibitem[Caprio et~al., 2024a]{caprio_cbdl_2024}
Caprio, M., Dutta, S., Jang, K.~J., Lin, V., Ivanov, R., Sokolsky, O., and Lee,
  I. (2024a).
\newblock {Credal Bayesian Deep Learning}.
\newblock {\em Transactions on Machine Learning Research}.

\bibitem[Caprio and Seidenfeld, 2023]{caprio_constriction_2023}
Caprio, M. and Seidenfeld, T. (2023).
\newblock Constriction for sets of probabilities.
\newblock In {\em Proceedings of the Thirteenth International Symposium on
  Imprecise Probability: Theories and Applications}.

\bibitem[Caprio et~al., 2024b]{caprio_credal_2024}
Caprio, M., Sultana, M., Elia, E., and Cuzzolin, F. (2024b).
\newblock {Credal Learning Theory}.
\newblock In {\em Proceedings of the $38^{th}$ Conference on Neural Information
  Processing Systems (NeurIPS 2024)}.

\bibitem[Deshmukh et~al., 2019]{deshmukh_generalization_2019}
Deshmukh, A., Lei, Y., Sharma, S., Dogan, U., Cutler, J., and Scott, C. (2019).
\newblock A generalization error bound for multi-class domain generalization.
\newblock Accessed via \url{https://arxiv.org/abs/1905.10392}.

\bibitem[Fisher, 1936]{iris_53}
Fisher, R.~A. (1936).
\newblock {Iris}.
\newblock UCI Machine Learning Repository.
\newblock {DOI}: https://doi.org/10.24432/C56C76.

\bibitem[{Le Cam}, 1953]{le-cam_asymptotic_1953}
{Le Cam}, L. (1953).
\newblock On some asymptotic properties of maximum likelihood estimates and
  related bayes' estimates.
\newblock {\em University of California publications in statistics}, 1(11).

\bibitem[Liu et~al., 2017]{liu_algorithm_2017}
Liu, T., Tao, D., Song, M., and Maybank, S.~J. (2017).
\newblock Algorithm-dependent generalization bounds for multi-task learning.
\newblock 39(2).

\bibitem[Lu et~al., 2024]{lu_ibcl_2024}
Lu, P., Caprio, M., Eaton, E., and Lee, I. (2024).
\newblock {IBCL: Zero-shot Model Generation under Stability-Plasticity
  Trade-offs}.
\newblock Accessed via \url{https://arxiv.org/abs/2305.14782}.

\bibitem[Maurer, 2006]{maurer_bounds_2006}
Maurer, A. (2006).
\newblock Bounds for linear multi-task learning.
\newblock {\em Journal of Machine Learning Research}, 2006.

\bibitem[Maurer and Pontil, 2013]{maurer_excess_2013}
Maurer, A. and Pontil, M. (2013).
\newblock Excess risk bounds for multitask learning with trace norm
  regularization.
\newblock {\em Journal of Machine Learning Research}, 30.

\bibitem[Maurer et~al., 2013]{maurer_sparse_2013}
Maurer, A., Pontil, M., and Romera-Paredes, B. (2013).
\newblock Sparse coding for multitask and transfer learning.
\newblock In {\em Proceedings of the $30^{\mathrm{th}}$ International
  Conference on Machine Learning}.

\bibitem[Redko et~al., 2019]{redko_survey_2019}
Redko, I., Morvant, E., Habrard, A., Sebban, M., and Bennani, Y. (2019).
\newblock {\em Advances in Domain Adaptation Theory}.
\newblock Elsevier.

\bibitem[Rudin, 1976]{rudin_principles_1976}
Rudin, W. (1976).
\newblock {\em Principles of Mathematical Analysis}.
\newblock McGraw Hill.

\bibitem[Sason and Verdu, 2015]{sason_upper_2015}
Sason, I. and Verdu, S. (2015).
\newblock Upper bounds on the relative entropy and rényi divergence as a
  function of total variation distance for finite alphabets.
\newblock {\em 2015 {IEEE} Information Theory Workshop - Fall ({ITW})}.

\bibitem[Shen and Wasserman, 2001]{shen_rates_2001}
Shen, X. and Wasserman, L. (2001).
\newblock Rates of convergence of posterior distributions.
\newblock {\em The Annals of Statistics}, 29(3).

\bibitem[Singh et~al., 2024]{singh_domain_2024}
Singh, A., Chau, S.~L., Bouabid, S., and Muandet, K. (2024).
\newblock Domain generalisation via imprecise learning.
\newblock In {\em Proceedings of the $41^{st}$ International Conference on
  Machine Learning (ICML 2024)}.

\bibitem[{van der Vaart}, 1998]{van-der-vaart_asymptotic_2000}
{van der Vaart}, A.~W. (1998).
\newblock {\em Asymptotic Statistics}.
\newblock Cambridge Series in Statistical and Probabilistic Mathematics.
  Cambridge University Press.

\bibitem[Walker, 2004]{walker_modern_2004}
Walker, S.~G. (2004).
\newblock Modern bayesian asymptotics.
\newblock {\em Statistical Science}, 19(1).

\bibitem[Walley, 1991]{walley_statistical_1991}
Walley, P. (1991).
\newblock {\em Statistical Reasoning with Imprecise Probabilities}.
\newblock Springer-Science+Business Media, B.V.

\bibitem[Zhang et~al., 2023]{zhang_survey_2023}
Zhang, W., Deng, L., Zhang, L., and Wu, D. (2023).
\newblock A survey on negative transfer.
\newblock {\em IEEE/CAA Journal of Automatica Sinica}, 10(2).

\end{thebibliography}


\begin{thebibliography}{}

\bibitem[Abdar et~al., 2021]{abdar_uncertainty_2021}
Abdar, M., Samami, M., Mahmoodabad, S.~D., Doan, T., Mazoure, B.,
  Hashemifesharaki, R., Liu, L., Khosravi, A., Acharya, U.~R., Makarenkov, V.,
  and Nahavandi, S. (2021).
\newblock Uncertainty quantification in skin cancer classification using
  three-way decision-based bayesian deep learning.
\newblock {\em Computers in Biology and Medicine}, 135.

\bibitem[Abdullah et~al., 2022]{abdullah_review_2022}
Abdullah, A.~A., Hassan, M.~M., and Mustafa, Y.~T. (2022).
\newblock A review on bayesian deep learning in healthcare: Applications and
  challenges.
\newblock {\em IEEE Access}, 10.

\bibitem[Angelopoulos and Bates, 2023]{angelopoulos_conformal_2023}
Angelopoulos, A.~N. and Bates, S. (2023).
\newblock Conformal prediction: A gentle introduction.
\newblock In {\em Foundations and Trends\textregistered{} in Machine Learning}.
  now Publishers Inc.

\bibitem[Baxter, 2000]{baxter_model_2000}
Baxter, J. (2000).
\newblock A model of inductive bias learning.
\newblock {\em Journal of Artificial Intelligence Research}, 12.

\bibitem[Ben-David et~al., 2010]{ben-david_theory_2010}
Ben-David, S., Blitzer, J., Crammer, K., Kulesza, A., Pereira, F., and {Wortman
  Vaughan}, J. (2010).
\newblock A theory of learning from different domains.
\newblock {\em Machine Learning}, 79.

\bibitem[Ben-David et~al., 2006]{ben-david_analysis_2006}
Ben-David, S., Blitzer, J., Crammer, K., and Pereira, F. (2006).
\newblock Analysis of representations for domain adaptation.
\newblock In {\em Proceedings of the 20th Conference on Neural Information
  Processing Systems (NIPS 2006)}.

\bibitem[Birrell et~al., 2022]{birrell_divergences_2022}
Birrell, J., Dupuis, P., Katsoulakis, M.~A., Pantazis, Y., and Rey-Bellet, L.
  (2022).
\newblock {$\left( f, \Gamma \right)$}-divergences: Interpolating between
  $f$-divergences and integral probability metrics.
\newblock {\em Journal of Machine Learning Research}, 23.

\bibitem[Blanchet et~al., 2024]{blanchet_distributionally_2024}
Blanchet, J., Li, J., Lin, S., and Zhang, X. (2024).
\newblock {Distributionally Robust Optimization and Robust Statistics}.
\newblock Accessed via \url{https://arxiv.org/abs/2401.14655}.

\bibitem[Caprio et~al., 2024a]{caprio_cbdl_2024}
Caprio, M., Dutta, S., Jang, K.~J., Lin, V., Ivanov, R., Sokolsky, O., and Lee,
  I. (2024a).
\newblock {Credal Bayesian Deep Learning}.
\newblock {\em Transactions on Machine Learning Research}.

\bibitem[Caprio et~al., 2024b]{caprio_credal_2024}
Caprio, M., Sultana, M., Elia, E., and Cuzzolin, F. (2024b).
\newblock {Credal Learning Theory}.
\newblock In {\em Proceedings of the $38^{th}$ Conference on Neural Information
  Processing Systems (NeurIPS 2024)}.

\bibitem[Deshmukh et~al., 2019]{deshmukh_generalization_2019}
Deshmukh, A., Lei, Y., Sharma, S., Dogan, U., Cutler, J., and Scott, C. (2019).
\newblock A generalization error bound for multi-class domain generalization.
\newblock Accessed via \url{https://arxiv.org/abs/1905.10392}.

\bibitem[Doob, 1949]{doob_application_1949}
Doob, J.~L. (1949).
\newblock Application of the theory of martingales.
\newblock In {\em Le Calcul des Probabilit\'{e}s et ses Applications}.

\bibitem[Fisher, 1936]{iris_53}
Fisher, R.~A. (1936).
\newblock {Iris}.
\newblock UCI Machine Learning Repository.
\newblock {DOI}: https://doi.org/10.24432/C56C76.

\bibitem[H\"{u}llermeier and Waegeman, 2021]{hullermeier_aleatoric_2021}
H\"{u}llermeier, E. and Waegeman, W. (2021).
\newblock Aleatoric and epistemic uncertainty in machine learning: an
  introduction to concepts and methods.
\newblock {\em Machine Learning}, 110.

\bibitem[Izbicki, 2025]{izbicki_machine_2025}
Izbicki, R. (2025).
\newblock {\em Machine Learning Beyond Point Predictions: Uncertainty
  Quantification}.

\bibitem[Liu et~al., 2017]{liu_algorithm_2017}
Liu, T., Tao, D., Song, M., and Maybank, S.~J. (2017).
\newblock Algorithm-dependent generalization bounds for multi-task learning.
\newblock 39(2).

\bibitem[Maurer, 2006]{maurer_bounds_2006}
Maurer, A. (2006).
\newblock Bounds for linear multi-task learning.
\newblock {\em Journal of Machine Learning Research}, 2006.

\bibitem[Maurer and Pontil, 2013]{maurer_excess_2013}
Maurer, A. and Pontil, M. (2013).
\newblock Excess risk bounds for multitask learning with trace norm
  regularization.
\newblock {\em Journal of Machine Learning Research}, 30.

\bibitem[Maurer et~al., 2013]{maurer_sparse_2013}
Maurer, A., Pontil, M., and Romera-Paredes, B. (2013).
\newblock Sparse coding for multitask and transfer learning.
\newblock In {\em Proceedings of the $30^{\mathrm{th}}$ International
  Conference on Machine Learning}.

\bibitem[Papamarkou et~al., 2024]{papamarkou_position_2024}
Papamarkou, T., Skoularidou, M., Palla, K., Aitchison, L., Arbel, J., Dunson,
  D., Filippone, M., Fortuin, V., Hennig, P., Hern\'{a}ndez-Lobato, J.~M.,
  Hubin, A., Immer, A., Karaletsos, T., Khan, M.~E., Kristiadi, A., Li, Y.,
  Mandt, S., Nemeth, C., Osborne, M.~A., Rudner, T. G.~J., R\"{u}gamer, D.,
  Teh, Y.~W., Welling, M., Wilson, A.~G., and Zhang, R. (2024).
\newblock Position: Bayesian deep learning is needed in the age of large-scale
  ai.
\newblock In {\em Proceedings of the 41st International Conference on Machine
  Learning (ICML 2024)}.

\bibitem[Peng et~al., 2019]{peng_bayesian_2019}
Peng, W., Ye, Z.-S., and Chen, N. (2019).
\newblock Bayesian deep-learning-based health prognostics toward prognostics
  uncertainty.
\newblock {\em IEEE Transactions on Industrial Electronics}, 67.

\bibitem[Redko et~al., 2019]{redko_survey_2019}
Redko, I., Morvant, E., Habrard, A., Sebban, M., and Bennani, Y. (2019).
\newblock {\em Advances in Domain Adaptation Theory}.
\newblock Elsevier.

\bibitem[Sale et~al., 2024]{sale_second-order_2024}
Sale, Y., Bengs, V., Caprio, M., and H\"{u}llermeier, E. (2024).
\newblock Second-order uncertainty quantification: A distance-based approach.
\newblock In {\em Proceedings of the 41st International Conference on Machine
  Learning (ICML 2024)}.

\bibitem[Sason and Verdu, 2015]{sason_upper_2015}
Sason, I. and Verdu, S. (2015).
\newblock Upper bounds on the relative entropy and rényi divergence as a
  function of total variation distance for finite alphabets.
\newblock {\em 2015 {IEEE} Information Theory Workshop - Fall ({ITW})}.

\bibitem[Shalev-Shwarz and Ben-David, 2014]{shalev-shwarz_understanding_2014}
Shalev-Shwarz, S. and Ben-David, S. (2014).
\newblock {\em Understanding Machine Learning: From Theory to Algorithms}.
\newblock Cambridge University Press.

\bibitem[Sloman et~al., 2025]{sloman_proxy_2025}
Sloman, S.~J., Martinelli, J., and Kaski, S. (2025).
\newblock Proxy-informed bayesian transfer learning with unknown sources.
\newblock In {\em Proceedings of the Forty-first Conference on Uncertainty in
  Artificial Intelligence (UAI 2025)}.

\bibitem[Suder et~al., 2023]{suder_bayesian_2023}
Suder, P.~M., Xu, J., and Dunson, D.~B. (2023).
\newblock Bayesian transfer learning.
\newblock Accessed via \url{https://arxiv.org/abs/2312.13484}.

\bibitem[Tripuraneni et~al., 2021]{tripuraneni_overparameterization_2021}
Tripuraneni, N., Adlam, B., and Pennington, J. (2021).
\newblock Overparameterization improves robustness to covariate shift in high
  dimensions.
\newblock In {\em Proceedings of the 35th Conference on Neural Information
  Processing Systems (NeurIPS 2021)}.

\bibitem[{van der Vaart}, 1998]{van-der-vaart_asymptotic_2000}
{van der Vaart}, A.~W. (1998).
\newblock {\em Asymptotic Statistics}.
\newblock Cambridge Series in Statistical and Probabilistic Mathematics.
  Cambridge University Press.

\bibitem[Wang et~al., 2019]{wang_characterizing_2019}
Wang, Z., Dai, Z., Poczos, B., and Carbonell, J. (2019).
\newblock Characterizing and avoiding negative transfer.
\newblock In {\em Proceedings of the IEEE/CVF Conference on Computer Vision and
  Pattern Recognition (CVPR)}.

\bibitem[Wasserman and Kadane, 1990]{wasserman_bayes_1990}
Wasserman, L.~A. and Kadane, J.~B. (1990).
\newblock Bayes' theorem for choquet capacities.
\newblock {\em The Annals of Statistics}, 18(3).

\bibitem[Zhang et~al., 2023]{zhang_survey_2023}
Zhang, W., Deng, L., Zhang, L., and Wu, D. (2023).
\newblock A survey on negative transfer.
\newblock {\em IEEE/CAA Journal of Automatica Sinica}, 10(2).

\end{thebibliography}
